\documentclass[twoside,11pt]{article}

\usepackage[letterpaper]{geometry}

\usepackage{natbib}
\usepackage{hyperref}       % hyperlinks
\usepackage{amsfonts}       % blackboard math symbols
\usepackage{xcolor}         % colors

\usepackage[boxed]{algorithm2e}
\usepackage{empheq}
\usepackage{graphicx}
\usepackage{tikz}
\usepackage{subcaption}
\usepackage{multirow}
\usepackage{bookmark}
\usepackage{amsmath,amsthm,amssymb,bbm}
\usepackage{thmtools}

\usepackage{amsmath}
\newcommand{\vect}[1]{\boldsymbol{#1}}

\def\E{\mathbb{E}}

\def\R{\mathbb{R}}
\def\S{\mathbb{S}}

\def\Var{\mathrm{Var}}
\def\half{\frac{1}{2}}

\def\cG{\mathcal{G}}
\def\cH{\mathcal{H}}

\def\cK{\mathcal{K}}

\def\cN{\mathcal{N}}

\def\act{\psi}
\def\loss{\mathrm{loss}}

\newenvironment{eqal}{
\begin{equation}
\begin{aligned}
}
{
\end{aligned}    
\end{equation}}

\newenvironment{eqal*}{
\begin{equation*}
\begin{aligned}
}
{
\end{aligned}    
\end{equation*}}

\newtheorem{theorem}{Theorem}
\newtheorem{lemma}[theorem]{Lemma}
\newtheorem{proposition}[theorem]{Proposition}
\newtheorem{assumption}{Assumption}
\newtheorem{definition}[assumption]{Definition}
\newtheorem{remark}[theorem]{Remark}

\newcommand{\BlackBox}{\rule{1.5ex}{1.5ex}} 
\newenvironment{proofs}{\par\noindent{\bf Proof Sketch:\ }}{\hfill\BlackBox\\[2mm]}

\title{Why Quantization Improves Generalization: NTK of Binary Weight Neural Networks}

\author{
    Kaiqi Zhang, Ming Yin, Yu-Xiang Wang
}
\date{}

\begin{document}

\maketitle

\begin{abstract}
    Quantized neural networks have drawn a lot of attention as they reduce the space and computational complexity during the inference. Moreover, there has been folklore that quantization acts as an implicit regularizer and thus can improve the generalizability of neural networks, yet no existing work formalizes this interesting folklore. 
    In this paper, we take the binary weights in a neural network as random variables under stochastic rounding, and study the distribution propagation over different layers in the neural network.
    We propose a \emph{quasi neural network} to approximate the distribution propagation, which is a neural network with continuous parameters and smooth activation function.
    %a smooth activation function \my{one can approximate BWNN with what? Quasi-NN? should complete this sentence.} 
    We derive the neural tangent kernel (NTK) for this quasi neural network, and show that the eigenvalue of NTK decays at approximately exponential rate, which is comparable to that of Gaussian kernel with randomized scale. 
    %it is similar to \my{strictly cover the RKHS of the Gaussian kernel?} a Gaussian kernel, and 
%its Reproducing kernel Hilbert space (RKHS) covers a strict subset of functions as the RKHS of NTK of its continuous counterpart.
    This in turn indicates that
    the \emph{Reproducing Kernel Hilbert Space} (RKHS) of a binary weight neural network covers a strict subset of functions compared with the one with real value weights. 
    We use experiments to verify that the quasi neural network we proposed can well approximate binary weight neural network. Furthermore, binary weight neural network gives a lower generalization gap compared with real value weight neural network, which is similar to the difference between Gaussian kernel and Laplace kernel.
    %From the empirical side, we can practically implement the binary weight NN through our approximated quasi neural network, which gets rid of the discrete optimization hurdle encountered by binary weight NN. Our empirical result shows the binary weight NN is as accurate as the Gaussian kernel, Laplace kernel and full precision NN. \my{this needs to be further modified.} 
\end{abstract}

\section{Introduction}
%TODO: figure: NTK, Laplacian, Gaussian

%Neural networks have achieved state-of-the-art performance on many machine learning tasks including computer vision, natural language processing, and recommendation systems. However, they often demand extensive computing resources in their training and inference.
%From the empirical side, 
It has been found that by quantizing the parameters in a neural network, the memory footprint and computing cost can be greatly decreased with little to no loss in accuracy \citep{gupta2015deep}. 
Furthermore, \citet{hubara2016binarized, courbariaux2015binaryconnect} argued that quantization serves as an implicit regularizer and thus should increase the generalizability of neural network comparing to its full precision version. However, there is no formal theoretical investigation of this statement to the best of our knowledge.

% The BinaryConnect \citep{courbariaux2015binaryconnect} algorithm is one of the most popular methods to train neural networks with quantized weights and it has achieved good performance on many tasks.
% However, \citet{bai2018proxquant} proved that the algorithm will fail to find the local minimum for certain type of target function.
% In this work, we study this method under overparameterized regime, and %TODO 
%in a certain circumstance\my{what circumstance? Explain briefly in a few words.}.
%This motivates us to dig deeper into this method and study BinaryConnect from a new angle,\my{what new angle? through stochastic rounding?} with the hope to shed a light on how BinaryConnect works and why quantized neural networks could generalize better. 

% Despite the empirical success of neural networks, it remains an open question why neural networks generalize well even when the number of parameters can be far larger than the number of training samples. 
Empirical results show that the traditional statistical learning techniques based on uniform convergence (e.g., VC-dimension \citep{blumer1989learnability}) do not satisfactorily explain the generalization ability of neural networks.
\citet{zhang2016understanding} showed that neural networks can perfectly fit the training data even if the labels are random, yet it generalized well when the data are not random. 
This seems to suggest that the model capacity of a neural network depends on not only the model, but also the dataset.
% This indicates the traditional statistical learning techniques based on uniform convergence (e.g., VC-dimension \citep{blumer1989learnability}) do not satisfactorily explain the empirical performance of overparameterized neural networks. 
Recent studies \citep{he2020recent} managed to understand the empirical performance in a number of different aspects, including modeling stochastic gradient (SGD) with stochastic differential equation (SDE) \citep{weinan2019mean}, studying the geometric structure of loss surface \citep{he2020piecewise}, and overparameterization -- a particular asymptotic behavior when the number of parameters of the neural network tends to infinity \citep{li2018over, choromanska2015loss, allen2018learning, arora2019fine}. 
Recently, it was proven that the training process of neural network in the overparameterized regime corresponds to kernel regression with Neural Tangent Kernel (NTK) \citep{jacot2018neural}. A line of work \citep{bach2017breaking, bietti2019inductive, geifman2020similarity, chen2020deep} further studied  Mercer's decomposition of NTK and proved that it is similar to a Laplacian kernel in terms of the eigenvalues.

In this paper, we propose modeling a two-layer binary weight neural network using a model with continuous parameters. Specifically, we assume the binary weights are drawn from the Bernoulli distribution where the parameters of the distribution (or the mean of the weights) are trainable parameters. We propose a \emph{quasi neural network}, which has the same structure as a vanilla neural network but has a different activation function, and prove one can analytically approximate the expectation of output of this binary weight neural network with this quasi neural network.
Using this model, our main contributions are as follows:
\begin{itemize}
    \item 
    %Under stochastic rounding, we define a quasi neural network modeling the mean of output of a binary weight neural network for given input and real-valued model parameters. 
    Under the overparameterized regime, we prove that the gradient computed from BinaryConnect algorithm is approximately an unbiased estimator of the gradient of the quasi neural network, hence such a quasi neural network can model the training dynamic of binary weight neural network. 
    \item We study the NTK of two-layer binary weight neural networks by studying the ``quasi neural network'', and show that the eigenvalue of this kernel decays at an exponential rate, in contrast with the polynomial rate in a ReLU neural network \cite{chen2020deep, geifman2020similarity}. We reveal the similarity between the Reproducing kernel Hilbert space (RKHS) of this kernel with Gaussian kernel, and it is a strict subset of function as the RKHS of NTK in a ReLU neural network. This indicates that the model capacity of binary weight neural network is smaller than that with real weights, and explains higher training error and lower generalization gap observed empirically.
    %\my{Need to be understood.}
\end{itemize}

\section{Related work}

\paragraph{Quantized neural networks.}
There is a large body of work that focuses on training neural networks with quantized weights \citep{marchesi1993fast, hubara2017quantized, gupta2015deep, liang2021pruning, chu2021mixed}, including considering radically quantizing the weights to binary \citep{courbariaux2016binarized, rastegari2016xnor} or ternary \citep{alemdar2017ternary} values, which often comes at a mild cost on the model's predictive accuracy. 
Despite all these empirical works, the theoretical analysis of quantized neural networks and their convergence is not well studied. Many researchers believed that quantization adds noise to the model, which serves as an implicit regularizer and makes neural networks generalize better \citep{hubara2016binarized, courbariaux2015binaryconnect}, but this statement is instinctive and has never been formally proved to the best of our knowledge.
One may argue that binary weight neural networks have a smaller parameter space than its real weight counterpart, yet \citet{ding2018universal} showed that a quantized ReLU neural network with enough parameters can approximate any ReLU neural network with arbitrary precision. 
%However, \citet{zhang2016understanding} found that the effective capacity of neural networks didn't explain its generalization performance.
These seemingly controversy results motivate us to find another way to explain the stronger generalization ability that is observed empirically.

\paragraph{Theory of deep learning and NTK.}
A notable recent technique in developing the theory of neural networks is the neural tangent kernel (NTK) \cite{jacot2018neural}. It draws the connection between an over-parameterized neural network and the kernel learning. This makes it possible to study the generalization of overparameterized neural network using more mature theoretical tools from kernel learning \cite{bordelon2020spectrum, simon2021neural}. 

The expressive power of kernel learning is determined by the RKHS of the kernel. Many researches have been done to identify the RKHS.
%whether NTK is similar to standard kernels. 
\citet{bach2017breaking, bietti2019inductive} studied the spectral properties of NTK of a two-layer neural network without bias. \citet{geifman2020similarity} further studied the NTK with bias and showed that the RKHS of two layer neural networks contains the same set of functions as RKHS of the Laplacian kernel. \citet{chen2020deep} expanded this result to arbitrary layer neural networks and showed that RKHS of arbitrary layer neural network is equivalent to Laplacian kernel. All these works are based on neural networks with real weights, and to the best of our knowledge, we are the first to study the NTK and generalization of binary weight neural networks. 

\section{Preliminary}

\subsection{Neural tangent kernel}
It has been found that an overparameterized neural network has many local minima. 
Furthermore, most of the local minima are almost as good as the global minima \citep{soudry2016no}.
As a result, in the training process, the model parameters often do not need to move far away from the initialization point before reaching a local minimum \citep{du2018gradient, du2019gradient, li2018learning}. 
This phenomenon is also known as lazy training \citep{chizat2018lazy}. 
This allows one to approximate a neural network with a model that is nonlinear in its input and linear in its parameters. 
Using the connection between feature map and kernel learning, the optimization problem reduces to kernel optimization problem. More detailed explanation can be found below:

%\my{how to understand this?}
% Consider an $L$-layer neural network with only fully connected layers and ReLU activation, and the activation function in the last layer is the identity function. This neural network can be written as  $y =  f_w (x)$,
% where $x$ is the input, $w$ is the trainable parameters, and $y$ is the prediction given by this neural network. Let $z$ be the ground truth label and $\cD$ be distribution of  the dataset.
% The update rule of the gradient descent algorithm can be expressed as:
% \begin{eqal*}
% w^{+} - w &= -\E_{(x, z)\sim \cD} \left[\frac{\partial \ell(f_w(x), z)}{\partial w}\right]
% = -\E_{(x, z)\sim \cD} \left[\frac{\partial \ell(y, z)}{\partial y} \nabla_w f_w(x)\right],
% \end{eqal*}
% where $ y = f_w(x)$. Furthermore, according to the chain rule of gradient, for any $\tilde x$ (potentially in the testing set):

Denote $\Theta$ as the collection of all the parameters in a neural network $f_\Theta$ before an iteration, and $\Theta^+$ as the parameters after this iteration. Let $in$ denote fixed distribution in the input space. In this paper, it is a discrete distribution induced by the training dataset. Using Taylor expansion, for any testing data $\tilde x$, let the stepsize be $\eta$, the first-order update rule of gradient descent can be written as ($l_{\text{loss}}(\cdot)$ be the differentiable loss function and the label is omitted)
\begin{eqal*}
    \Theta^+ - \Theta
    &=\eta \E_{x \sim in} \left[\nabla_\Theta \loss(f_\Theta(x))\right]\\
    %---------------
    &=\eta \E_{x \sim in} \left[\nabla_\Theta f_\Theta(x) \ \loss'(f_\Theta(x))\right]\\
    f_{\Theta^+}(x') - f_{\Theta}(x')
    %---------------
    &=\eta \nabla_\Theta f_\Theta(x') \cdot \E_{x\sim in} \left[ \nabla_\Theta f_\Theta(x) \  \loss'(f_\Theta(x))\right]\\
    %---------------
    &=\eta \E_{x\sim in} \left[\nabla_\Theta f_\Theta(x')\cdot \nabla_\Theta f_\Theta(x) \  \loss'(f_\Theta(x))\right]\\
    %---------------
    &:=\eta \E_{x\sim in} \left[\cK(x, x')\  \loss'(f_\Theta(x))\right].
\end{eqal*}
%We can view this equation as kernel learning with  
This indicates that the learning dynamics of overparameterized neural network is equivalent to kernel learning with kernel defined as
$$
\cK(x, x') = \nabla_\Theta^\top f_\Theta(x)\cdot \nabla_\Theta f_\Theta(x').
$$

%In an overparameterized neural network,
which is called the neural tangent kernel (NTK).
As the width of the hidden layers in this neural network tends to infinity,
this kernel convergences to its expectation over $\Theta$ \citep{jacot2018neural}.

\subsection{Exponential kernel}
A common class of kernel functions used in machine learning is the exponential kernel, which is a radial basis function kernel with the general form  
\begin{equation*}
    \cK(x, x') = \exp(-(c\|x-x'\|)^\gamma),
\end{equation*}
where $c > 0$ and $\gamma \geq 1$ are constants. When $\gamma=1$, this kernel is known as the Laplacian kernel, and when $\gamma=2$, 
%\yw{You wrote when $c=1$, I thought it was a typo. kaiqi: yes, thanks for pointing out}, 
it is known as the Gaussian kernel. 

According to Moore-Aronszajn theorem, each symmetric positive definite kernel uniquely induces a Reproducing kernel Hilbert space (RKHS). 
RKHS determines the functions that can be learned using a kernel.
It has been found that the RKHS of NTK in a ReLU neural network is the same as Laplacian kernel \citep{geifman2020similarity, chen2020deep}, and the empirical performance of a neural network is close to that of kernelized linear classifiers with exponential kernels in many datasets \citep{geifman2020similarity}.

\subsection{Training neural networks with quantized weights}
%To quantize a real number to discrete values, two methods can be used: deterministic rounding and stochastic. Deterministic rounding always rounds a number to the closest discrete values, while stochastic rounding rounds to the two closest values randomly so that the expectation after quantization equals the original real number. Deterministic rounding introduces smaller quantization errors and stochastic rounding is unbiased.

%During the training of a neural network, the gradient step can be smaller than the quantization step of weights especially when the stepsize is small. To avoid losing these updates, BinaryConnect (BC) Algorithm has been proposed \cite{courbariaux2015binaryconnect}. 
Among various methods to train a neural network, BinaryConnect (BC) \citep{courbariaux2015binaryconnect} is often one of the most efficient and accurate method.
The key idea is to introduce a real-valued buffer $\theta$ and use it to  accumulate the gradients. The weights will be quantized just before forward and backward propagation, which can benefit from the reduced computing complexity. The update rule in each iteration is
\begin{equation}
    \theta^{+} \leftarrow \theta - \eta \frac{\partial \tilde f_w(x)}{\partial w},\quad w^{+} \leftarrow Quantize(\theta^{+}), 
\label{eq:bc}
\end{equation}
where $Quantize(\cdot): \R \rightarrow \{-1, 1\}$ denotes the quantization function which will be discussed in Section \ref{sec:problem}, $w$ denotes the binary (or quantized) weights, $\theta$ and $\theta^+$ denote the real valued buffer before and after an iteration respectively, $\eta$ is the learning rate, and $f_w(\cdot)$ denotes the neural network with parameter $w$. Here the gradients are computed by taking $w$ as if they are real numbers. The detailed algorithm can be founded in Section~\ref{sec:binaryconnect}. 

%Two quantization funtions can be used: deterministic rounding and stochastic rounding. Deterministic rounding always rounds a real number to the quantized one with the smallest distance to it, and stochastic rounding rounds a real number to two nearest values randomly so that the mean after quantization equals the real value. The former has smaller variance by is biased, while the latter has larger variance but is unbiased. It has been found the the latter method often leads to better empirical accuracy. 

\section{Approximation of binary weight neural network}
\subsection{Notations}
In this paper, we use $w_{\ell,ij}$ to denote the binary weights in the $\ell$-th layer, $\theta_{\ell, ij}$ to denote its real-valued counterpart, and $b_{\ell,i}$ to denote the (real valued) bias. $\Theta$ is the collection of all the real-valued model parameters which will be specified in Section \ref{sec:problem}. %$\{\theta_{\ell_1, ij}, w_{\ell_2, ij}, b_{\ell,i}\}$, where $\ell_1 \in \{\ell: \ell\textrm{-th layer is quantized}\}$, and $\ell_2 \in \{\ell: \ell\textrm{-th layer is not quantized}\}$.
The number of neurons in the $\ell$-th hidden layer is $d_{\ell}$, the input to the $\ell$-th linear layer is $\vect x_{\ell}$ and the output is $\vect y_{\ell}$.  $d$ denote the number of input features. %As a short alias, we denote $\vect x=\vect x_{1}$ and $\vect y = \vect y_{L}$.
Besides, we use $\vect x$ to denote the input to this neural network, $y$ to denote the output and $z$ to denote the label.

We focus on the mean and variance under the randomness of stochastic rounding. Denote 
\begin{eqal*}
    \mu_{\ell, i} &:= \E[x_{\ell,i}|\vect x, \Theta],&
    \sigma_{\ell, i}^2 &:= \Var[x_{\ell, i}|\vect x, \Theta],\\
    \nu_{\ell, i} &:= \E[y_{\ell, i}|\vect x, \Theta], &
    \varsigma_{i,\ell}^2 &:=\Var[y_{\ell, i}|\vect x, \Theta], &
    \bar y &:= \E[y|\Theta].
\end{eqal*} 
We use $\act(x) = \max(x, 0)$ to denote ReLU activation function, and $in$ to denote the (discrete) distribution of training dataset. $\E_{in}[\cdot] := \E_{(\vect x, z) \sim in}[\cdot]$ denotes the expectation over training dataset, or ``sample average''. 
We use bold symbol to denote a collection of parameters or variables  $\vect{w}_2 = [w_{2,j}], \vect{b}_2 = [b_{2,j}], \vect{\nu}_1 = [\nu_{1,j}],\vect \theta_{1} = [\theta_{1, ij}], i \in [d_1],j \in [d_2]$.
%and $\|\cdot\|_{in}^2 := \E_{in}\|\cdot\|^2$. %as the norm defined under training dataset.

\subsection{Problem statement}
\label{sec:problem}
\begin{figure}
    \centering
    \begin{subfigure}{.48\textwidth}
    \begin{tikzpicture}[
        roundnode/.style={circle, draw=green!60, fill=green!5, very thick, minimum size=7mm},
        squarednode/.style={rectangle, draw=red!60, fill=red!5, very thick, minimum size=5mm},
    ]   
        
        \node [anchor=north] at (0, 1) {Input};
        \node [roundnode] (x01) at (0, 0) {};
        \node [roundnode] (x02) at (0, -0.8) {};
        \node       at (0, -1.6) {$\dots$};
        \node [roundnode] (x03) at (0, -2.4) {};
        \node at (0, -3) {$x$};
        \node at (0.7, -3) {$w_{0}$};
        
        \node [roundnode] (x11) at (1.5, 0) {};
        \node [roundnode] (x12) at (1.5, -0.8) {};
        \node                   at (1.5, -1.6) {$\dots$};
        \node [roundnode] (x13) at (1.5, -2.4) {};
        \node at (1.5, -3) {$x_1$};
        
        \node [anchor=north, text width=1cm] at (2.5, 1.2) {First layer};
        \node [anchor=north] at (3.8, 1) {Activation};
        \node [roundnode] (x21) at (3, 0) {};
        \node [roundnode] (x22) at (3, -0.8) {};
        \node                   at (3, -1.6) {$\dots$};
        \node [roundnode] (x23) at (3, -2.4) {};
        
        \node [roundnode] (x31) at (4.5, 0) {};
        \node [roundnode] (x32) at (4.5, -0.8) {};
        \node                   at (4.5, -1.6) {$\dots$};
        \node [roundnode] (x33) at (4.5, -2.4) {};
        
        \node at (2.3, -3) {$w_{1}$};
        \node at (3, -3) {$y_{1}$};
        \node at (3.8, -3) {$\act(\cdot)$};
        \node at (4.5, -3) {$x_{2}$};
        \node at (5.3, -3) {$w_{2}$};
        ;
        \node [anchor=north, text width=1cm] at (5.5, 1.2) {Second layer};
        \node [anchor=north] at (5.8, 0) {Output};
        \node [roundnode] (x4)  at (6, -1.2) {};
        \node at (6, -2) {$y$};
        \foreach \x in {1,2,3} {
            \foreach \y in {1,2,3} {
                \draw[->] (x0\x) -- (x1\y);
            }
            \foreach \y in {1,2,3} {
                \draw[->] (x1\x) -- (x2\y);
            }
            \draw[->] (x2\x) -- (x3\x);
            \draw[->] (x3\x) -- (x4);
        }
    \end{tikzpicture}
    \caption{Binary weight neural network we focus on.}
    \end{subfigure}
    %---------------------------------
    \begin{subfigure}{.48\textwidth}
    \begin{tikzpicture}[
        roundnode/.style={circle, draw=green!60, fill=green!5, very thick, minimum size=7mm},
        squarednode/.style={rectangle, draw=red!60, fill=red!5, very thick, minimum size=5mm},
    ]   
        
        \node [anchor=north] at (0, 1) {Input};
        \node [roundnode] (x01) at (0, 0) {};
        \node [roundnode] (x02) at (0, -0.8) {};
        \node       at (0, -1.6) {$\dots$};
        \node [roundnode] (x03) at (0, -2.4) {};
        \node at (0, -3) {$x$};
        \node at (0.7, -3) {$w_{0}$};
        
        \node [roundnode] (x11) at (1.5, 0) {};
        \node [roundnode] (x12) at (1.5, -0.8) {};
        \node                   at (1.5, -1.6) {$\dots$};
        \node [roundnode] (x13) at (1.5, -2.4) {};
        \node at (1.5, -3) {$x_1$};
        
        \node [anchor=north, text width=1cm] at (2.5, 1.2) {First layer};
        \node [anchor=north] at (3.8, 1) {Activation};
        \node [roundnode] (x21) at (3, 0) {};
        \node [roundnode] (x22) at (3, -0.8) {};
        \node                   at (3, -1.6) {$\dots$};
        \node [roundnode] (x23) at (3, -2.4) {};
        
        \node [roundnode] (x31) at (4.5, 0) {};
        \node [roundnode] (x32) at (4.5, -0.8) {};
        \node                   at (4.5, -1.6) {$\dots$};
        \node [roundnode] (x33) at (4.5, -2.4) {};
        \node at (2.3, -3) {$\theta_{1}$};
        \node at (3, -3) {$\nu_{1}$};
        \node at (3.8, -3) {$\tilde\act(\cdot)$};
        \node at (4.5, -3) {$\mu_{2}$};
        \node at (5.3, -3) {$w_{2}$};
        \node [anchor=north, text width=1cm] at (5.5, 1.2) {Second layer};
        \node [anchor=north] at (5.8, 0) {Output};
        \node [roundnode] (x4)  at (6, -1.2) {};
        \node at (6, -2) {$\bar y$};
        \foreach \x in {1,2,3} {
            \foreach \y in {1,2,3} {
                \draw[->] (x0\x) -- (x1\y);
            }
            \foreach \y in {1,2,3} {
                \draw[->] (x1\x) -- (x2\y);
            }
            \draw[->] (x2\x) -- (x3\x);
            \draw[->] (x3\x) -- (x4);
        }
    \end{tikzpicture}
    \caption{Quasi neural network.}
    \end{subfigure}
\end{figure}

In this work, we target on stochastic quantization \citep{dong2017learning}, which often yields higher accuracy empirically compared with deterministic rounding \citep{courbariaux2015binaryconnect}. This also creates a smooth connection between the binary weights in a neural network and its real-valued parameters.
%In this work we're focusing on stochastic rounding.

Let $w_{\ell, ij} = Quantize(\theta_{\ell, ij}), \theta_{\ell, ij} \in [-1, 1]$ be the binary weights from stochastic quantization function, which satisfy Bernoulli distribution:
\begin{equation}
w_{\ell,ij}= \left\{
    \begin{array}{cl}
       +1,  & \textrm{ with probability }p_{\ell, ij}=\frac{\theta_{\ell, ij}+1}{2}, \\
      -1,  & \textrm{ with probability }1-p_{\ell, ij}.
    \end{array}
    \right.
\label{eq:bernoulli}
\end{equation}
This relationship leads to $\E[w_{\ell,ij}|\theta_{\ell,ij}]=\theta_{\ell,ij}$.

We focus on a ReLU neural network with one hidden layer and two fully connect layers, which was also studied in \citet{bach2017breaking, bietti2019inductive} except quantization. Besides, we add a linear layer (``additional layer'') in front of this neural network to project the input to an infinite dimension space. We randomly initialize the weights in this layer and leave it fixed (not trainable) throughout the training process. 
%The purpose of this layer is that it allows us to use central limit theorem to study the distribution of the input of the hidden layer. 
%view the addition layer as embedded in the raw input data.
%this layer doesn't exist, but as the dimension of input is often large ($\approx 10^3$), 
%central limit theorem is still a reasonable approximation.
Furthermore, we quantize the weights in the first fully connect layer $w_{1, ij}$ and add a real-valued buffer $\theta_{1, ij}$ which determines the distribution of $w_{1, ij}$ as in  \eqref{eq:bernoulli}, and leave the second layers not quantized.
It is a common practice to leave the last layer not quantized, because this often leads to better empirical performance.  If the second layer is quantized as well, the main result of this paper will not be changed. This can be easily checked by extending Lemma \ref{lemma:linear} into the second layer.

\begin{remark}
    In many real applications, e.g. computer vision, the dimension of data are often very large ($\approx 10^3$) while they are laying in the lower dimension linear subspace, so we can take the raw input in these applications as the output of the additional layer, and the NN in this case is a two-layer NN where the first layer is quantized.
\end{remark}

The set of all the real-valued parameters is $\Theta = \{\theta_{\ell_1, ij}, w_{\ell_2, ij}, b_{\ell,i}\}$.
The neural network can be expressed as
\begin{eqal*}
    x_{1,i} &= \frac{1}{\sqrt{d}} \sum_{k=1}^{d} w_{0, ki}x_k + b_{0, i}, \forall i \in [d_1];&
    %x_{2,j} &= \act_1(y_{1,j}),\\
    y_{1,j} &= \sqrt{\frac{c}{d_1}} \sum_{i=1}^{d_1} w_{1, ij}x_{1,i}+b_{1,j}, \forall j \in [d_2];\\
    x_{2,j} &= \act(y_{1,j}), \forall j \in [d_2]; &
    y &= \frac{1}{\sqrt{d_2}}\sum_{j=1}^{d_2} w_{2,j} x_{2,j} + b_2.
\end{eqal*}
We follow the typical setting of NTK papers \cite{geifman2020similarity} in initializing the parameters except the quantized parameters. 
As for the quantized parameters, we only need to specify the real-valued buffer of the weights in the first layer $\theta_{1, ij}$.

\begin{assumption}
    \label{ass:init}
    We randomly initialize the weights in the ``additional layer'' and second linear layer independently as $w_{0, ki}, w_{2, j} \sim \cN(0, 1)$, 
    and initialize all the biases to 0. 
    The real-valued buffer of the weights are initialized independently identical with zero mean, variance $\Var[\theta]$ and bounded in $[-1, 1]$.
\end{assumption}
\begin{remark}
    Our theory applies to any initial distribution of $\theta_{1, ij}$ as long as it satisfies the constraint above.
    One simple example is the uniform distribution in $[-1, 1]$, which has variance $\Var[\theta] = 1/3$. 
\end{remark}

%We choose this architecture because it has been found that by leaving the first and last layer not quantized, the accuracy of the neural network can be significantly increased. This also makes the output of quantized layer to converge asymptotically to Gaussian distribution conditioned on the real-valued model parameters.

\subsection{Quasi neural network}
\label{sec:quasinn}
Given a fixed input and real-value model parameters $\Theta$, under the randomness of stochastic rounding, the output of this binary weight neural network is a random variable. Furthermore, as the width of this neural network tends to infinity,
%As the number of input features to a linear layer tends to infinity, given a fixed input and real-value model parameters $\Theta$ and under the randomness of stochastic rounding, 
the output of a linear layer tends to Gaussian distribution according to central limit theorem (CLT). 
We propose a method to determine the distribution of output and using the model parameters. 
Specifically, we give a closed form equation to compute the mean and variance of output of all the layers $\mu_\ell, \sigma_\ell, \nu_\ell, \varsigma_\ell$, and then marginalize over random initialization of $\Theta$ to further simplify this equation. 
We prove that $\varsigma_\ell$ converges to a constant almost surely using the law of large number (LLN), and simplify the expression by replacing them with the constant. This allows us to compute $\mu_\ell, \nu_\ell$ using a neural-network-style function for  given $\Theta$. We call this function \emph{quasi neural network}, which is given below:
\begin{eqal}
    %\nu_{1,j} &= \sqrt{\frac{c}{d_{1}}} \sum_{i=1}^{d_{1}}w_{1,ij} x_{i} + \beta b_{1,j}, &
    %\mu_{2, j} &= \tilde\act_{1}(\nu_{1,j}), \\
    x_{1,i} &= \frac{1}{\sqrt{d}} \sum_{k=1}^{d} w_{0,ki}x_k + b_{0,i}, 
    \forall i \in [d_1];&
    \nu_{1,j} &= \sqrt{\frac{c}{d_{1}}}
    \sum_{i=1}^{d_{1}}\theta_{1,ij} x_{1, j} + \beta b_{1,i}, 
    \forall j \in [d_2];\\
    \mu_{2, j} &= \tilde\act(\nu_{1,j}), 
    \forall j \in [d_2]; &
    \bar y &= \frac{1}{\sqrt{d_2}}
    \sum_{j=1}^{d_{2}}w_{2,j} \mu_{2,j} + \beta b_{2}.
\label{eq:2nn}
\end{eqal}

% We will give the closed-form expression to $\tilde\act(\cdot)$ in Section \ref{sec:var}. 
In \autoref{sec:meanvar}, we study the distribution of the output of each layer in a binary weight neural network (BWNN) conditioned on the set of real-valued parameter $\Theta$. 
In \autoref{sec:var}, we prove that the conditioned variance of the output of the first linear layer studied above converges almost surely to a constant which does not depend on the data (input). 
This simplifies the expression computed in \autoref{sec:meanvar} to the form of quasi neural network \eqref{eq:2nn}, and also give a closed-form expression to $\tilde\act(\cdot)$ in \eqref{eq:2nn}.
In \autoref{sec:grad}, we prove that conditioned on the set of real-valued parameter, the expectation of the gradients of BWNN equals the gradient of quasi neural network on the overparameterization limit. 
This indicates that the training dynamics of BWNN at initialization is the same as training the quasi neural network directly.
The training dynamics beyond initialization are discussed in \autoref{sec:training}.
Before jumping to the proof, we make the following assumptions:

\begin{assumption} 
\label{assum:nooverflow} 
After training the binary weight neural network as in \eqref{eq:bc}, all the real-valued weights $\theta_{\ell,ij}$ stay in the range $[-1, 1]$. 
\end{assumption}
Based on this assumption, we can ignore constraints that $\theta_{\ell,ij} \in [-1, 1]$ and the projected gradient descent reduces to gradient descent. 
%\my{Maybe to write this as ``Based on this assumption, we can ignore ...''} 
Because of the lazy training property of the overparameterized neural network, the model parameters $\theta_{\ell,ij}$ stay close to the initialization point during the training process, so this assumption can be satisfied by initializing $\theta_{\ell,ij}$ with smaller absolute value and/or applying weight decay during training.
On the other hand, a common trick in a quantized neural network is to change the quantization level gradually during the training process to avoid (or reduce) overflow. With this trick, Assumption \ref{assum:nooverflow} are often naturally satisfied, but it introduces the quantization level as a trainable parameter.
%TODO

\begin{assumption}
\label{ass:norm}
The Euclidean norm of the input is 1:
\begin{equation*}
    \|\vect x\|_2=1, \forall \vect x \in \R^d.
\end{equation*}
\end{assumption}
This is a common assumption in studying NTK \citep{bach2017breaking, bietti2019inductive}, and can be satisfied by normalizing the input data.

%The detailed explanation is given in the following two steps:
\subsubsection{Conditioned distribution of the outputs of each layer}
    % Computing the mean and variance of the outputs of each layer conditioned on the model parameters.}
\label{sec:meanvar}

First we recognize that as the model parameters $\Theta$ are initialized randomly, there are ``bad'' initialization that will mess up our analysis.
For example, all of $\theta_{1, ij}$ are initialized to 1 (or -1) while they are drawn from a uniform distribution. Fortunately, as the width $d_1, d_2$ grows to infinite, the probability of getting into these ``bad'' initialization is 0. We make this statement formal in the following part.

\begin{definition}
    ``Good Initialization''. We call the set of parameters is a ``Good Initialization'' $\Theta \in \cG $ if it satisfies:
    \begin{itemize}
        \item $\displaystyle \forall k, k' \in [d], \lim_{d_1 \rightarrow \infty}\frac{1}{d_1}\sum_{i=1}^{d_1} w_{0, ki} w_{0, k'i} = \delta_{k, k'},$
        \item $\displaystyle \forall k, k', k'' \in [d], \lim_{d_1 \rightarrow \infty}\frac{1}{d_1}\sum_{i=1}^{d_1} |w_{0, ki} w_{0, k'i} w_{0, k'i}| \leq \sqrt{\frac{8}{\pi}},$
        \item $\forall k, k' \in [d], \forall j \in [d_2]$ for some finite $d_2$, $\displaystyle \lim_{d_1 \rightarrow \infty}\frac{1}{d_1}\sum_{i=1}^{d_1} w_{0, ki} w_{0, k'i} \theta_{1, ij}^2 = \Var[\theta]\delta_{k, k'}$,
    \end{itemize}
    where
    \begin{empheq}[left={\delta_{k, k'}=\empheqlbrace}]{align*}
        &1 \quad k = k'\\
        &0 \quad k \neq k'.
    \end{empheq}
\end{definition}

\begin{proposition}
    \label{prop:asconv}
    Under the assumption that all the parameters are initialized as in \autoref{ass:init}, the probability of getting ``Good Initialization'' is 1: 
    \begin{equation*}
        P_\Theta (\Theta \in \cG) = 1.
    \end{equation*}
\end{proposition}

\begin{lemma}
\label{lemma:linear}
    On the limit $d_1 \rightarrow \infty$, and any given $x$ conditioned on any fixed $\Theta \in \cG$, for any $j$,
    %with high probability over random initialization of $\Theta$, 
    the distribution of $y_{1, j}$ converge to \emph{Gaussian distribution} with mean $\nu_{1, j}$ and variance $\varsigma_{1, j}^2$ which can be computed by:
    \begin{eqal}
        y_{1, j}|\Theta &\rightarrow \cN(\nu_{1, j}, \varsigma_{1, j}^2), &
        \nu_{1, j} &= \sqrt{\frac{c}{d_1}} \sum_{i=1}^{d_1} \theta_{1,ij}x_{1, i} + b_{1, j}, &
        \varsigma_{1, j}^2 &= \frac{c}{d_1} \sum_{i=1}^{d_1} (1-\theta_{1, ij}^2) x_{1, i}^2
    \label{eq:linear}
    \end{eqal}
\end{lemma}

This lemma can be proved by Lyapunov central limit theorem and sum of expectation. See \autoref{sec:linearproof} for the details.

% \subsection{ReLU layer}
% \label{sec:relu}
\begin{lemma}
\label{lemma:relu}
%Under Assumption \ref{assump:gauss}, 
Assume that the input to a ReLU layer $y_{1,i}$ satisfy Gaussian distribution with mean $\nu_{1,i}$ and variance $\varsigma_{1,i}^2$ conditioned on $\Theta$,
%\my{further think about whether this gaussian assumption can be relaxed}
%and each element $x_i$ has mean $\mu_i$ and variance $\sigma_i^2$,
\begin{equation*}
    y_{1,j} | \Theta \sim \cN (\nu_{1,j}, \varsigma_{1,j}^2).
\end{equation*}
Denote 
\begin{equation}
g_{j} = \varphi\left(\frac{\nu_{j}}{\varsigma_{j}} \right), 
s_{j} = \Phi\left(\frac{\nu_{j}}{\varsigma_{j}}\right),
\label{eq:phi}
\end{equation}
where $\varphi(x)$ denotes standard Gaussian function and $\Phi(x)$ denotes its integration:
\begin{eqal*}
\varphi(x) &= \sqrt{\frac{1}{2\pi}}\exp\left(-\half x^2 \right),&
\Phi(x) &= \int_{-\infty}^{x} \varphi(y)dy.
\end{eqal*}
Then the output $x_{2,j}$ has mean $\mu_{2,i}$ and variance $\sigma_{2,i}^2$ conditioned on $\Theta$, with 
\begin{eqal}
\mu_{2,j} &:= \E[x_{2,j}|\Theta] =  g_{j}\varsigma_{1,j} + s_{j} \nu_{1,j},\\
\sigma_{2,i}^2 &:= \Var[x_{2,j}|\Theta] = %s_i(\mu_i^2+\sigma_i^2) + g_i \mu_i \sqrt{\frac{\sigma_i^2}{2\pi}}
(\varsigma_{1,i}^2 +\nu_{1,i}^2) s_{i} + \nu_{1,i}\sigma_{1,i} g_{1,i} - \nu_{1,j}^2 .
\label{eq:brelu}
\end{eqal}
\end{lemma}

From \autoref{lemma:linear} we know that on the limit $d_1 \rightarrow \infty$, conditioned on $\Theta$ and $\vect x$, for any $j$, $y_{1, j}$ converge to Gaussian distribution. 
From continuous mapping theorem, the distribution of $x_{2, j}$ converge to that shown in \autoref{lemma:relu} so its mean $\mu_{2, j} $ and variance $\sigma_{2, j}$ converge to that computed in \autoref{lemma:relu}.

Equations \eqref{eq:linear} and \eqref{eq:brelu} provide a method to calculate the mean and variance of output conditioned on the input and real-valued model parameters and allow us to provide a closed-form equation of quasi neural network. We will simplify this equation in \autoref{sec:var}.

\subsubsection{Convergence of conditioned variance}
\label{sec:var}
In this part, we assume that the model parameters satisfy ``Good Initialization'', which is almost surely on the limit $d_1 \rightarrow \infty$ as is proven in \autoref{prop:asconv}, 
%remind that $\theta_{1, ij}$ are initialized independently with zero mean and variance $\Var[\theta]$, we 
and study the distribution of $\nu_{1, j}$ and $\varsigma_{1, j}$.

\begin{theorem}
\label{thm:var}
%Under random initialization of $\Theta$, 
If the parameters are ``Good Initialization'' $\Theta \in \cG$, on the limit $d_1 \rightarrow \infty$, for any finite $d_2$,
% the mean of the output of the fist layer conditioned on $\Theta$, aka 
$\nu_{1, j}$ converges to Gaussian distribution which are independent of each other, and
%  variance of output of the first linear layer conditioned on $\Theta$, aka 
$\varsigma_{1, j}^2$ converges a.s. to 
\begin{equation*}
    \tilde\varsigma_1^2 
    % = \frac{c}{d_1}(1-\E[\theta_{1,ij}^2])\E[{x_{1, i}}^2]
    = \frac{c}{d}(1-\Var[\theta]).
\end{equation*}
%and $\mu_{2, i}, \sigma_{2, i}$ converges to that computed by replacing $\varsigma_{1, j}$ in \eqref{eq:brelu} with $\tilde\varsigma_1$.
\end{theorem}

With this approximation, we can replace the variance $\varsigma_{1,i}$ in equation (\ref{eq:brelu}) with $\tilde \varsigma_1$ and leave the mean of output in the linear layer as the only variable in the quasi neural network. Formal proof can be found in Section \ref{sec:prfvar}.
Note the propagation function in the linear layer (the first equation in (\ref{eq:linear})) is also a linear function in $x$ and $\theta$. 
%This reduces the quasi neural network to a neural network with activation function in layer $\ell$ as 
This motivates us to compute $\bar y$ using a neural network-like function as is given in \eqref{eq:2nn}, where $\tilde \act(\cdot)$ is 
\begin{equation}
    \tilde\act(\nu_{1, j}) = \E[\act(y_{1, j}) | \nu_{1, j}] = \tilde\varsigma_1 \phi\left( \frac{\nu_{1, i}}{\tilde\varsigma_1} \right) + \nu_{1, j} \Phi\left(\frac{\nu_{1, j}}{\tilde\varsigma_{1}}\right).
\label{eq:quasiact}
\end{equation}
This equation gives a closed-form connection between the mean of output of neural network $\bar y$ and the real-valued model parameter $\Theta$, and allows up to apply existing tools for analyzing neural networks with real-valued weight to analysis binary weight neural network.
Its derivative in the sense of Calculus is:
\begin{equation}
    \tilde\act'(\nu_{1, j}) 
    =\Phi\left(\frac{\nu_{1, j}}{\tilde \varsigma_{1}}\right).
\label{eq:quasiactgrad}
\end{equation}
The proof of derivative can be found in Section \ref{sec:phip}. %Here we are not claiming the meaning of this derivative in the sense of Malliavin calculus. %We will find a more practical meaning to this derivative in the next section.

%This finished the derivative of quasi nerual network.
% It taked $\theta_{ell, ij}$ as weight matrix, and takes $\tilde\act_\ell$ as the activation functions

\subsubsection{Gradient of quasi neural network}
\label{sec:grad}
%In this part, we come up with the closed-form approximation of the gradients with respect to weights and bias, and further verify that the gradient of variance is neglectable compared with mean so we can take them as constants in backpropagation. 
%The closed-form of gradients of quasi neural network \eqref{eq:quasiactgrad} is inefficient to compute empirically because it contains Gaussian function and its integration. 
% Because we are not claiming the meaning of the derivative \ref{eq:quasiactgrad} in the sense of distribution, its practical meaning is still not obvious.
In this part, we compute the gradients using binary weights as in BinaryConnect Algorithm, and make sense of the gradient in \eqref{eq:quasiactgrad} by proving that it is the expectation of gradients under the randomness of stochastic rounding.

\begin{theorem}
\label{thm:nngrad}
%For a two-layer neural network whose number of input features is $d$ and whose number of hidden features is $m$, 
%For the same neural network ,
%mentioned in Lemma \ref{lemma:nngrad}, 
The expectation of gradients to output with respect to weights computed by sampling the quantized weights equals the gradients of ``quasi neural network'' defined above in \eqref{eq:2nn} on the limit $d_2 \rightarrow \infty, d_1\rightarrow \infty$,
\begin{eqal*}
    \label{eq:gradconv}
    \frac{\partial \bar y}{\partial \theta_{1,ij}} &= \E\left[ \frac{\partial  y}{\partial w_{1,ij}}\middle|\Theta\right],&
    %(1 + O(\frac{1}{\sqrt{n}}))\\% 
    %+O\left(\frac{1}{d_2\sqrt{d_3}}\right),\\
    \frac{\partial \bar y}{\partial b_{1,j}} &= \E\left[ \frac{\partial y}{\partial b_{1,j}}\middle|\Theta\right],&
    \frac{\partial \bar y}{\partial w_{2,j}} &= \E\left[ \frac{\partial y}{\partial w_{2,j}}\middle|\Theta\right].
\end{eqal*}
\end{theorem}

\begin{theorem}
\label{thm:lossgrad}
For MSE loss, 
$
    \loss(y) = \half (y-z)^2,
    %\frac{dl(y)}{dy} = y - z,
$
where $z$ is the ground-truth label, the gradient of the loss converges on the limit $d_2 \rightarrow \infty, d_1\rightarrow \infty$, %under overparameterization $d_1, d_2 \rightarrow \infty$
\begin{equation*}
\begin{aligned}
    \frac{\partial \loss(\bar y)}{\partial \theta_{1,ij}} 
    &= \E\left[ \frac{\partial \loss(y)}{\partial w_{1,ij}}\middle|\Theta \right], &
    %+ O\left(\frac{\sqrt{d_2} + \sqrt{d_3}}{d_2d_3}\right),\\
    \frac{\partial \loss(\bar y)}{\partial b_{1,j}} 
    &= \E\left[ \frac{\partial \loss(y)}{\partial b_{1,j}}\middle|\Theta\right], \\
    %+ O\left( \frac{1}{d_3} \right),\\
    \frac{\partial \loss(\bar y)}{\partial w_{2,j}} 
    &= \E\left[ \frac{\partial \loss(y)}{\partial w_{2,j}}\middle|\Theta\right].
    %+ O\left( \frac{1}{d_3} \right).
\end{aligned}
\end{equation*}
\end{theorem}
%Note that $\frac{\partial y}{\partial w_{2,ij}} \asymp \frac{1}{\sqrt{d_2d_3}}$, so the second term smaller than the first term by a factor of $\frac{1}{\sqrt{d_2}}$. 
In other words, the BinaryConnect algorithm provides an unbiased estimator to the gradients for the quasi neural network on this limit of overparameterization.
%expectation of gradients approximates gradient of gradient of expectation.

%we can approximate the variance of pre-activation as a constant in both forward and backward propagation: in forward propagation, we can replace it with a constant, and in backward propagation, its gradient can be ignored. 
%This indicates that BinaryConnect is equivalent to optimizing quasi neural network. NTK is the correct way to study training dynamics.%TODO 

%With this approximation, we can represent the distribution of activations with their mean and build a quasi neural network. The fully connected layer behaves in the same way as a standard neural network, and the activation function in the $l$-th layer is:

Theorem \ref{thm:var} and Theorem \ref{thm:lossgrad}  conclude that for an infinite wide neural network, the BinaryConnect algorithm is equivalent to training quasi neural network with stochastic gradient descent (SGD) directly. Furthermore, this points out the gradient flow of BinaryConnect algorithm and allows us to study this training process with neural tangent kernel (NTK).

\subsubsection{Asymptotics during training}
\label{sec:training}
So far we have studied the distribution of output during initialization. To study the dynamic of binary weight neural network during training, one need to extend these results to any parameter during training $\Theta(t), t \in [0, T]$. Fortunately, motivated by \cite{jacot2018neural}, we can prove that as $d_1, d_2 \rightarrow \infty$, the model parameters $\Theta(t)$ stays asymptotically close to initialization for any finite $T$, so-called ``lazy training'', so the above results apply to the entire training process. 
\begin{lemma}
For all $T$ such that $\int_{t=0}^T \|\bar y(t)-z\|_{in} dt$ stays stochastically bounded, as $d_2 \rightarrow \infty, d_1\rightarrow \infty$, $ 
\| {\boldsymbol w}_{2}(T) - {\boldsymbol w}_{2}(0)\|,
\| {\boldsymbol b}_{1}(T) - {\boldsymbol b}_{1}(0)\|,
\| {\boldsymbol\theta_{1}}(T) - {\boldsymbol\theta_{1}}(0)\|_{F}
$
are all stochastically bounded,  $\|{\boldsymbol \nu}_{1}(t) - {\boldsymbol \nu}_{1}(0) \|$ and $\int_{t=0}^T \big\|\frac{\partial {\boldsymbol \nu}_{1}(t)}{\partial t}\big\|dt$ is stochastically bounded for all $x$.
    
%\textcolor{red}{TODO}
\label{thm:lazy1}
\end{lemma}
The proof can be found in \autoref{sec:lazyproof1}. Note that $\|\vect w_{2}\|=O(\sqrt{d_2}), \|\vect \theta_{1}\|_F=O(\sqrt{d_1d_2})$, this results indicates that as $d_2 \rightarrow \infty$, the varying of the parameter is much smaller than the initialization, or so-called ``lazy training''. Making use of this result, we further get the follow result:
\begin{lemma}
    Under the condition of \autoref{thm:lazy1}, Lyapunov's condition holds for all $T$ so $y_{1, j}$ converges to Gaussian distribution conditioned on the model parameters $\Theta(T)$. Furthermore, $\varsigma_{1, j}(T) \rightarrow \varsigma_{1, t}(0)$, which equals $\tilde\varsigma_{1}$ almost surely.
\label{thm:lazy2}
\end{lemma}

% \begin{theorem}
%     Under the condition of \autoref{thm:lazy1}, for any $x, x'$, the kernel defined as \autoref{eq:qntkdef} satisfies
%     \begin{equation*}
%         \frac{\partial}{\partial t} \cK_{BWNN}(x, x') \rightarrow 0.
%     \end{equation*}
% \label{thm:lazy3}
% \end{theorem}
% \textcolor{red}{TODO: proof}
The proof can be found in \autoref{sec:lazyproof2}. This result shows that 
% the kernel is almost a constant throughout the training process and 
the analysis in \autoref{sec:quasinn} applies to the entire training process, and allows us to study the dynamics of binary weight neural network using quasi neural network.

\section{Capacity of Binary Weight Neural Network}
As has been found in \cite{jacot2018neural}, the dynamics of an overparameterized neural network trained with SGD is equivalent to kernel gradient descent where the kernel is NTK. As a result, the effective capacity of a neural network is equivalent to the RKHS of its NTK. In the following part, we will study the NTK of binary weight neural network using the approximation above, and compare it with Gaussian kernel.

\subsection{NTK of three-layer binary weight neural networks}
\label{sec:ntk}
We consider the NTK binary weight neural network by studying this ``quasi neural network'' defined as
\begin{eqal}
    \cK_{BWNN} (x, x') &
    %=\E_{\Theta}\left< \frac{\partial \bar y}{\partial \Theta}, \frac{\partial \bar y'}{\partial \Theta} \right>
    = \sum_{i=1, j=1}^{d_1, d_2} \frac{\partial \bar y}{\partial \theta_{1, ij}}  \frac{\partial \bar y'}{\partial \theta_{1,ij}} 
    + \sum_{j=1}^{d_2} \frac{\partial \bar y}{\partial b_{1,j}} \frac{\partial \bar y'}{\partial b_{1,j}}
    + \sum_{j=1}^{d_2} \frac{\partial \bar y}{\partial w_{2,j}} \frac{\partial \bar y'}{\partial w_{2,j}},
    %&\hspace{1.1cm} + \left.\sum_j \frac{\partial \E[y]}{\partial a_{j}} \frac{\partial \E[y']}{\partial a_{j}}\right],
\label{eq:qntkdef}
\end{eqal}
where $\Theta:=\{w_{1, ij}, b_{1,j}, b_{2,j}\}$ denotes all the trainable parameters.
% In the following discussion, we study its expectation over random initialization $\E\cK_{BWNN}$, because the kernel converges to its expectation almost surely.
We omitted the terms related to $b_2$ (which is a constant) in this equation.

First prove that the change of kernel asymptotically converges to 0 during training process.
% As for the kernel during training, we have the following conclusion, which shows that the kernel during training converges to the one while initialization:
\begin{theorem}
    \label{thm:lazy3}
    Under the condition of \autoref{thm:lazy1}, $\cK(x, x')(T) \rightarrow \cK(x, x')(0)$ at rate $1/\sqrt{d_2}$ for any $x, x'$.
\end{theorem}
The proof can be found in \autoref{sec:prooflazy3}.
Using Assumption~\ref{ass:norm}, we confine the input on the hypersphere $\S^{d-1}=\{x \in \R^d: \|x\|_2=1\}$. One can easily tell that it is positive definite, so we can apply Mercer's decomposition \cite{minh2006mercer} to it. 
% The remaining question is to find the basis and eigenvalues to this kernel.

To find the basis and eigenvalues to this kernel, we apply spherical harmonics decomposition to this kernel, which is common among studying of NTK \citep{bach2017breaking, bietti2019inductive}: 
\begin{equation}
    \cK_{BWNN}(x, x')=\sum_{k=1}^\infty u_k \sum_{j=1}^{N(d, k)} Y_{k, j}(x) Y_{k, j}(x'),
\label{eq:spdecom}
\end{equation}
where $d$ denotes the dimension of $x$ and $x'$,  $Y_{k,j}$ denotes the spherical harmonics of order $k$. This suggests that NTK of binary weight neural network and exponential kernel can be spanned by the same set of basis function. The key question is the decay rate of $u_k$ with $k$.

\begin{theorem}
\label{thm:exp}
NTK of a binary weight neural network can be decomposed using  \eqref{eq:spdecom}. If $k \gg d$, then
\begin{equation}
\mathrm{Poly}_1(k)C^{-k} \leq u_k \leq \mathrm{Poly}_2(k) C^{-k}.
\label{eq:qntk}
\end{equation}
where $\mathrm{Poly}_1(k)$ and $\mathrm{Poly}_2(k)$ denote polynomials of $k$, and $C$ is a constant.
% If $1 \ll k \ll d$, 
% \begin{equation}
%     A_1\left(\frac{C_1k}{d^2}\right) ^{k} \leq u_k \leq A_2\left(\frac{C_2k}{d^2}\right) ^{k}
% \end{equation}
% where $A_1, A_2, C_1, C_2$ are constants.

\end{theorem} 

In contrast, \citet{geifman2020similarity} shows that for NTK in the continuous space, it holds that 
$$
  C_1 k^{-d} \leq u_{k} \leq C_2 k^{-d},
$$
with   constants $C_1$ and $C_2$.
Because its decay rate is slower than that of the binary weight neural network, its RKHS covers a strict superset of functions  \citep{geifman2020similarity}.

\begin{proofs} 
We first compute NTK of quasi neural network, which depends on the distribution of $\mu_{1, j}$. As is shown in Theorem \ref{thm:var}, $\mu_{1, j}$ converge to Gaussian distribution on the limit of infinite wide neural network. 
To find the joint distribution of $\mu_{1, j}$ and $\mu'_{1, j}$ given arbitrary two inputs $x, x'$,
we combine the first linear layer in the quasi neural network with the ``additional layer'' in front of it (the first two equations in \eqref{eq:2nn}). This allows up to reparameterize $\mu_{1, j}$ as
\begin{equation*}
    \mu_{1, j} = \langle w_j, x \rangle,
\end{equation*}
where $w_j \sim \cN(0, \frac{c\Var[\theta]}{d}I)$ denotes the fused weight.
A key component in computing the NTK has the form
\begin{equation*}
    \E[\act(\mu_1)\act(\mu_1')]
    =\E[\act(\langle w, x\rangle)\act'(\langle w, x\rangle)]
    =\E_{\|w\|}\E[\act(\langle w, x\rangle)\act'(\langle w, x\rangle)|\|w\|].
\end{equation*}
The second equation comes from the law of total expectation. We use 2-norm in this expression. The inner expectation is equivalent to integration on a sphere, and can be computed by applying sphere harmonics decomposition to $\act(\cdot)$. The squared norm of the fused weight $\|w\|^2$ satisfy Chi-distribution, and we use momentum generating function to finish computing.
\end{proofs}

\subsection{Comparison with Gaussian Kernel}

% where $u_k$ satisfy \cite{minh2006mercer}
% \begin{equation}
%     u_k = \exp\left(-\frac{2}{\xi^2}\right)\xi^{d-2}I_{k+d/2-1}\left(\frac{2}{\xi^2}\right) \Gamma\left(\frac{d}{2}\right).
% \end{equation}
% Here $I$ denotes the modified Bessel function of the first kind. It has a factorial decay rate with respect to $k$:
% \begin{equation*}
% \begin{aligned}
%     &\left(\frac{2e}{\xi^2}\right)^k\frac{A_1}{(2k+d-2)^{k+d/2}} < u_k  < \left(\frac{2e}{\xi^2}\right)^k\frac{A_2}{(2k+d-2)^{k+d/2}}.
% \end{aligned}
% \end{equation*}
% Here $A_1, A_2$ are constants that depend on $\xi$ and $d$.  
Even if the input to a neural network $x$ is constrained on a unit sphere, the first linear layer (together with the additional linear layer in front of it) will project it to the entire $\R^d$ space with Gaussian distribution. In order to simulate that, we define a kernel by randoming the scale of $x$ and $x'$ beforing taking them into a Gaussian kernel.
\begin{eqal*}
    \cK_{RGauss}(x, x') = \E_\kappa[\cK_{Gauss}(\kappa x, \kappa x')],
\end{eqal*}
where $\cK_{Gauss}(x, x')=\exp\left(-\frac{\|x-x'\|^2}{\xi^2}\right)$ is a Gaussian kernel, $\kappa \sim \chi_{d}$ satisfy Chi distribution with $d$ degrees of freedom. This scaling factor projects a random vector uniformly distributed on a unit sphere to Gaussian distributed. The corresponding eigenvalue satisfy
\begin{eqal}
    A_1 C^{-k} \leq u_k \leq A_2 C^{-k},
\label{eq:gaussdecay}
\end{eqal}
where $A_1, A_2, C$ are constants that depend on $\xi$. The dominated term in both \eqref{eq:qntk} and \eqref{eq:gaussdecay} have an exponential decay rate $C^{-k}$, which suggests the similarity between NTK of binary weight neural network and Gaussian kernel. In comparison, \citet{bietti2019inductive, geifman2020similarity} showed that the eigenvalue of NTK decay at rate $k^{-d}$, which is slower that binary weight neural network or Gaussian kernel. Furthermore, Aronszajn's inclusion theorem suggests $\cH_{\cK_{BWNN}} \subset \cH_{\cK_{NN}}$, where $\cK_{NN}$ denotes the NTK of real-valued weight neural network. In other words, the expressive power of binary weight neural network is weaker than its real valued counterpart on the limit that the width goes to infinity. Binary weight neural networks are less venerable to noise thanks to the smaller expressive power at the expense of failing to learn some ``high frequency'' components in the target function. This explains that binary weight neural network often achieve lower training accuracy and smaller generalization gap compared with real weight neural network.

\section{Numerical result}
\label{sec:exp}
\subsection{Quasi neural network}
\begin{figure}
    \begin{tikzpicture}
        \node [anchor=north west] at (0,0.25) {\includegraphics[width=0.235\textwidth]{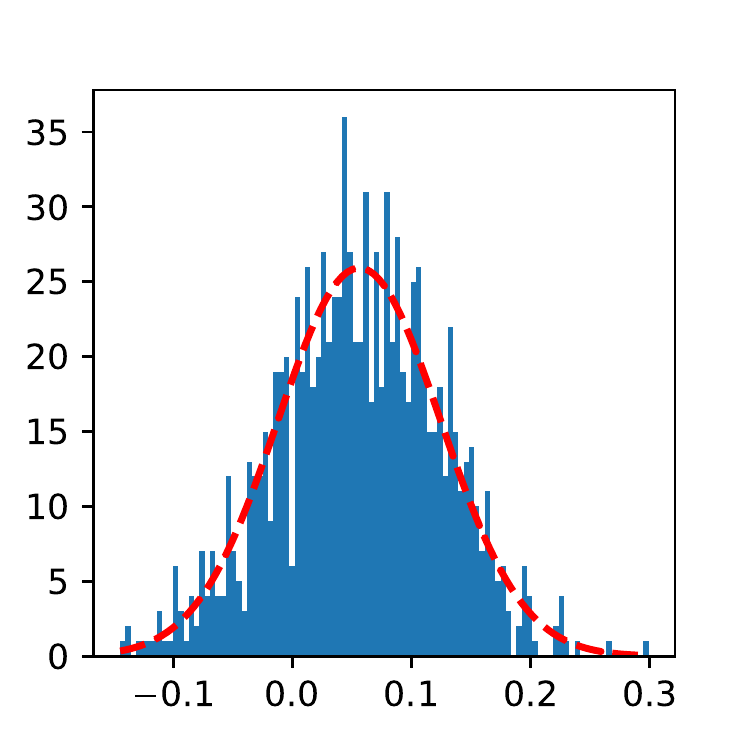}};
        \node at (0.03\textwidth, -0.24\textwidth) {\small(a)};
        \node [anchor=north west] at (0.24\textwidth,0.25) {\includegraphics[width=0.235\textwidth]{dist_pretrain.pdf}};
        \node at (0.28\textwidth,-0.24\textwidth) {\small(b)};
        \node [anchor=north west] at (0.48\textwidth,0) {\includegraphics[width=0.25\textwidth]{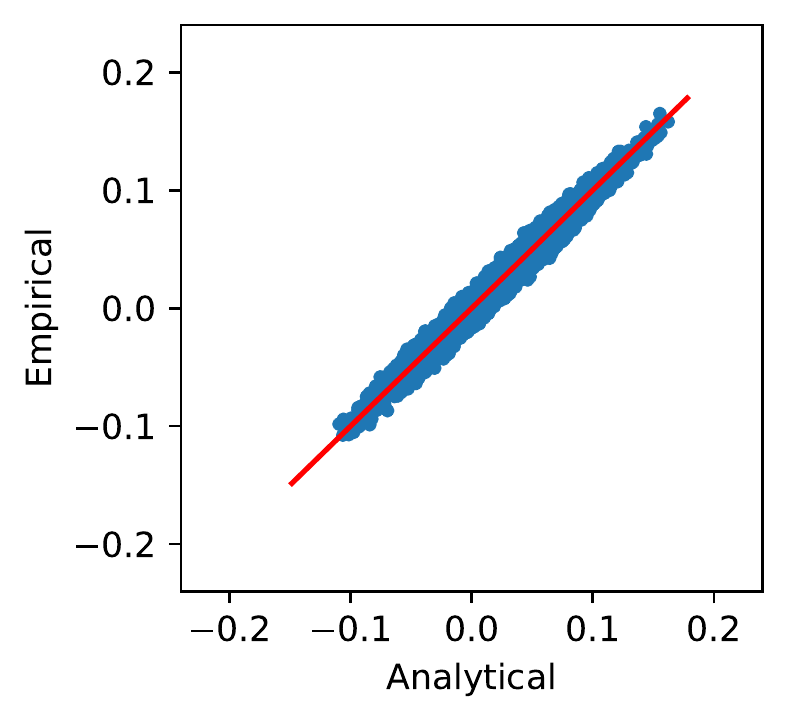}};
        \node at (0.55\textwidth,-0.24\textwidth) {\small(c)};
        \node [anchor=north west] at (0.75\textwidth,0) {\includegraphics[width=0.25\textwidth]{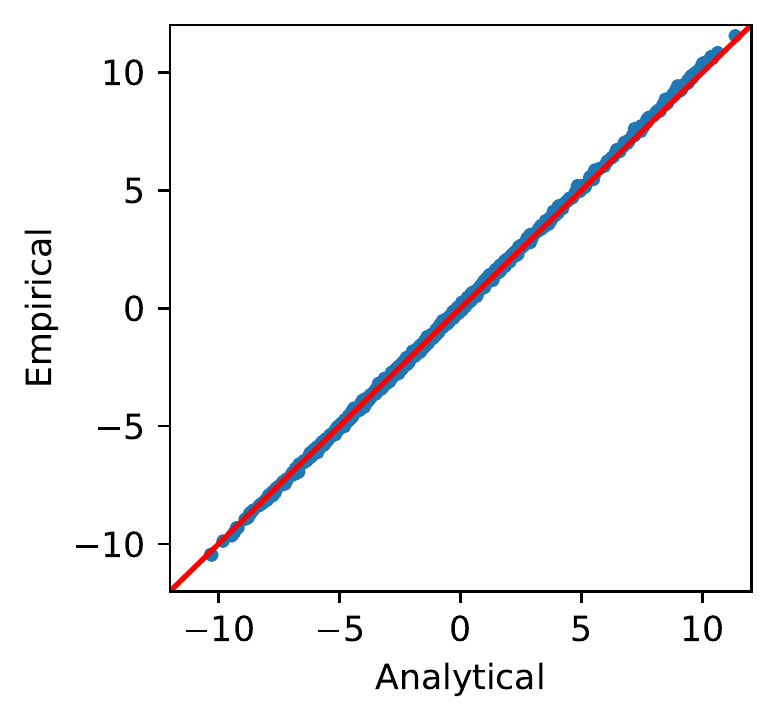}};
        \node at (0.80\textwidth,-0.24\textwidth) {\small(d)};
    \end{tikzpicture}
    
    % \centering
    % \begin{subfigure}{0.24\textwidth}
    % \includegraphics[width=\textwidth]{dist_pretrain.png}
    % \caption{}
    % \end{subfigure}
    % \begin{subfigure}{0.24\textwidth}
    % \includegraphics[width=\textwidth]{dist_posttrain.png}
    % \caption{}
    % \end{subfigure}
    % \begin{subfigure}{0.24\textwidth}
    % \includegraphics[width=\textwidth]{mean_pretrain.png}
    % \caption{}
    % \end{subfigure}
    % \begin{subfigure}{0.24\textwidth}
    % \includegraphics[width=\textwidth]{mean_posttrain.png}
    % \caption{}
    % \end{subfigure}
    \caption{Approximation of quasi neural network. (a)(b): before (a) and after (b) training, histogram of output under fixed model parameter (blue), and fitted with Gaussian distribution (red). (c)(d): $\E(y|\Theta)$ computed from quasi neural network (horizontal axis) and by Monte Carlo (Vertical axis). The red line shows $y=x$.}
    \label{fig:qnn}
\end{figure}
In this part, we empirically verify the approximation of quasi neural network by comparing the inference result of quasi neural network with that achieved by Monte Carlo. 
The architecture is the same as that mentioned in Section \ref{sec:problem}, with 1600 hidden neurons. 
We train this neural network on MNIST dataset \citep{lecun1998gradient} by directly applyng gradient descent to the quasi neural network. To reduce overflow, we add weight decay of 0.001 during training.
Figure \ref{fig:qnn}(a)(b) shows the histogram of output under stochastic rounding before and after training. 
We arbitrarily choose one input sample from the testing set and get 1000 samples under different stochastic rounding. 
This result supports our statement that the distribution of pre-activation (output of linear layer) conditioned on real-valued model parameters converge to Gaussian distribution.
Figure \ref{fig:qnn}(c)(d) compares the mean of output by quasi neural network approximation (horizontal axis) with that computed using Monte Carlo (vertical axis). 
These alignments further supports our method of approximating binary weight neural network with quasi neural network.   

\subsection{Generalization gap}
\subsubsection{Toy dataset}
We compare the performance of the neural network with/without binary weight and kernel learning using the same set of 90 small scale UCI datasets with less than 5000 data points as in \citet{geifman2020similarity, arora2019harnessing}.  
%We report the training accuracy and testing accuracy (mean $\pm$ std) of both vanilla neural network (NN) and binary weight neural network (BWNN) in Table \ref{tab:uci}. 
%On 53.33\% of the 90 datasets, Gaussian kernel achieves smaller generalization gap than Laplace kernel while on 37.78\% datasets the former has a larger generalization gap than the latter. Similarly, on 67.78\% of the dataset, BWNN achieves a smaller generalization gap than NN while only 26.67\% dataset we got the opposite relationship. 
We report the training accuracy and testing accuracy of both vanilla neural network (NN) and binary weight neural network (BWNN) in Figure \ref{fig:uci}. 
To further illustrate the difference, we list the paired T-test result of neural network (NN) against binary weight neural network (BWNN), and Gaussian kernel (Gaussian) against Laplace kernel (Laplace) using in Table \ref{tab:uci}. 
In this table, t-stats and p-val denotes the t-statistic and two-sided p-value of the paired t-test between two classifiers, and $<$ and $>$ denotes the percentage of dataset that the first classifier gets lower or higher testing accuracy or generalization bound (training accuracy - testing accuracy), respectively.

As can be seen from the results, although the Laplacian kernel gets higher training accuracy than the Gaussian kernel, its testing accuracy is almost the same as the latter one. 
In other words, the former has smaller generalization gap than the latter which can also be observed in Table \ref{tab:uci}. Similarly, a neural network gets higher training accuracy than a binary weight neural network but gets similar testing accuracy. 

%We observe that a binary weight neural network gets almost the same testing accuracy as a neural network while the generalization gap is smaller. Similarly, a kernel regression with Gaussian kernel gets almost the same testing accuracy as with Laplace kernel while the generalization gap is smaller.

\begin{figure}[t]
    \centering
    \begin{tikzpicture}[font=\footnotesize]
        \node [anchor=north west] at (0,0) {\includegraphics[width=0.48\textwidth]{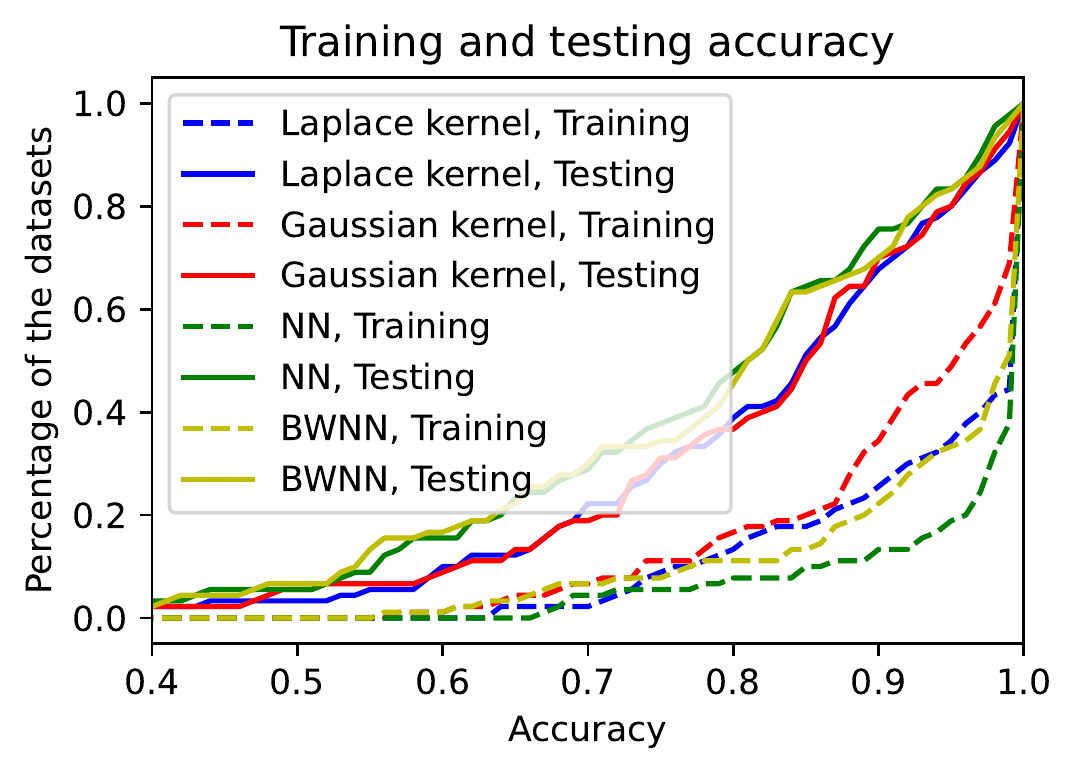}};
        \node at (0.25\textwidth, -0.36\textwidth) {(a)};
        \node [anchor=north west] at (0.5\textwidth,0) {\includegraphics[width=0.48\textwidth]{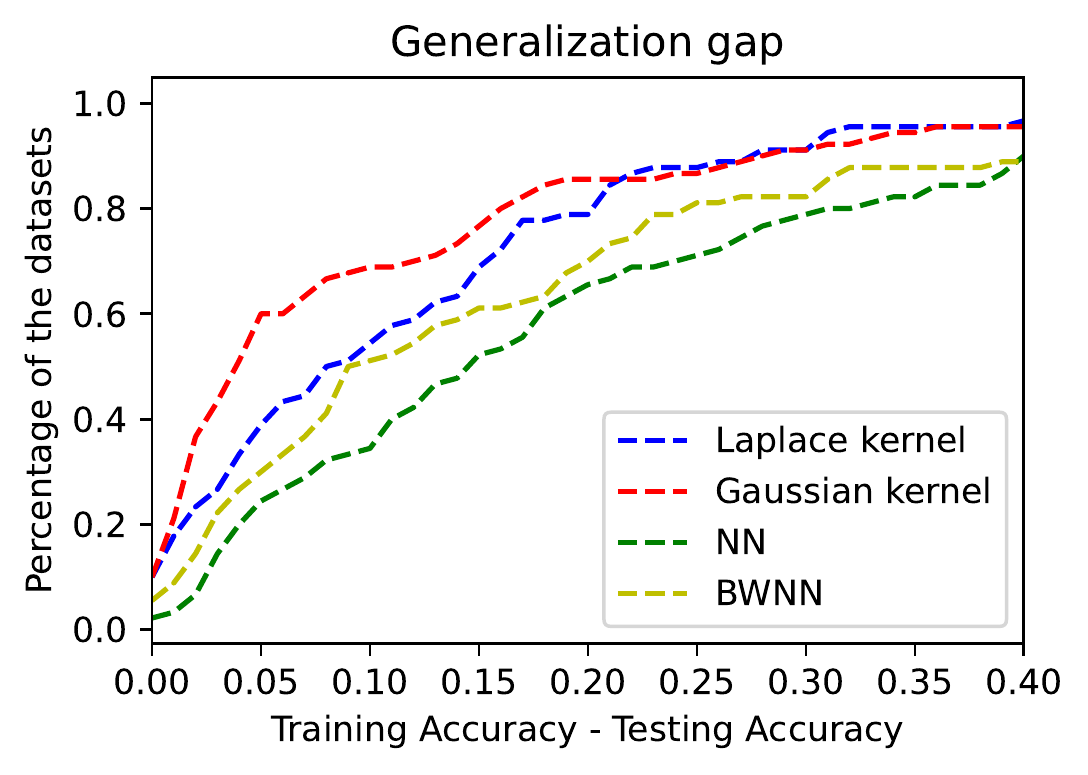}};
        \node at (0.75\textwidth,-0.36\textwidth) {(b)};
    \end{tikzpicture}
    \caption{Accuracy and generalization gap on selected 90 UCI datasets. The lines show the accuracy metric of a classifier from the lowest to the highest against their percentiles of datasets.}
    \label{fig:uci}
\end{figure}

\begin{table}[t]
    \centering
    \caption{Pairwise performance comparison on selected 90 UCI datasets.}
    \vspace{0.2cm}
    \begin{tabular}{c|cccc|cccc}
        \hline
        \multirow{2}{*}{Classifier} & \multicolumn{4}{c|}{Testing}  & \multicolumn{4}{c}{Training-Testing}\\
        & t-stats & p-val & $<$ & $>$ &  t-stats & p-val & $<$ & $>$\\
        \hline
        NN-BWNN & 0.7471 & 0.4569 & 53.33\% & 41.11\% & 4.034 & 0.000 & 26.67\% & 67.77\%\\ % $1.15\times 10^{-4}$ 
        Laplace-Gaussian & 0.4274 & 0.6701 & 51.11\% & 33.33\% & 3.280 & 0.001& 37.78\% & 53.33\% \\ % $1.48\times 10^{-3}$ 
        \hline
    \end{tabular}
    \label{tab:uci}
\end{table}

\subsubsection{MNIST-like dataset}
\begin{figure*}[t]
    \centering
    \begin{tikzpicture}
        \node[anchor=south west] at (0, 0){\includegraphics[width=\textwidth]{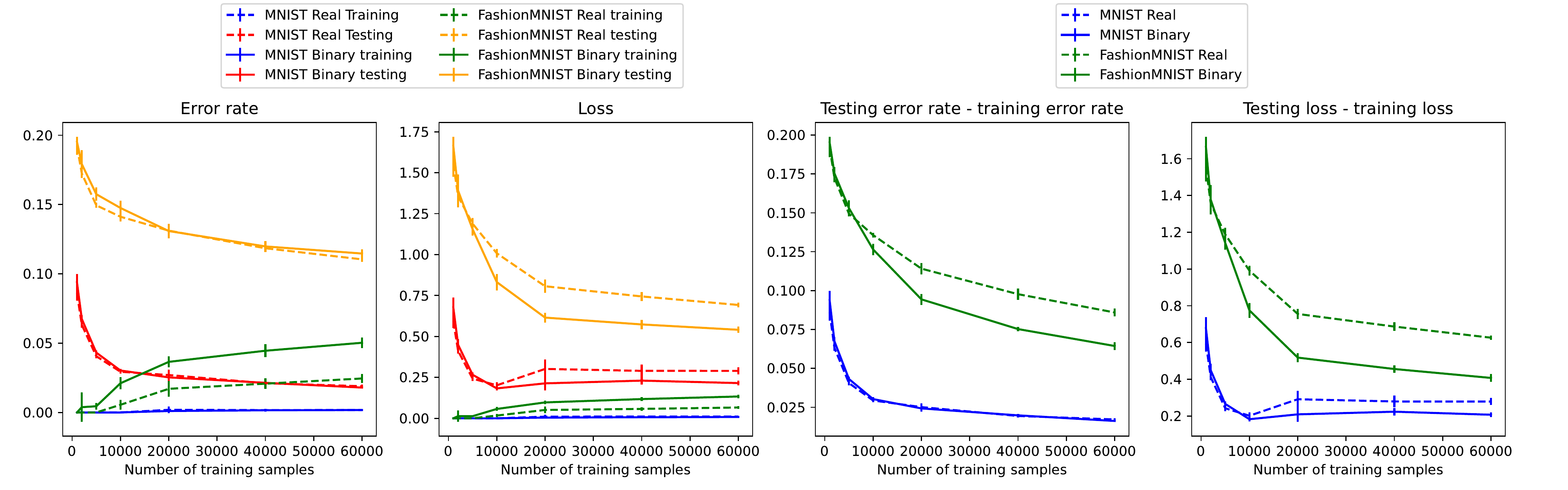}};
        \node at (0.03\textwidth, 0.2) {(a)};
        \node at (0.28\textwidth, 0.2) {(b)};
        \node at (0.53\textwidth, 0.2) {(c)};
        \node at (0.78\textwidth, 0.2) {(d)};
    \end{tikzpicture}
    \caption{Training/testing error rate and loss of neural networks with/without binary weight. (a) Training and testing error rate. (b) Training and testing loss. (c) Testing error rate - Training error rate. (d) Testing loss - Training loss. }
    \label{fig:mnist}
\end{figure*}
We compare the performance of neural networks with binary weights (Binary) with its counterpart with real value weights (Real). We take the number of training samples as a parameter by random sampling the training set and use the original test set for testing. The experiments are repeated 10 times and the mean and standard derivation is shown in Figure \ref{fig:mnist}. In the MNIST dataset, the performance of neural networks with or without quantization is similar. This is because MNIST \citep{lecun1998gradient} is simpler and less vulnerable to overfitting. On the other hand, the generalization gap with weight quantized is much smaller than without it in FashionMNIST \citep{xiao2017online} dataset, which matches our prediction.

\section{Discussion}
\label{sec:diss}
In this paper, we propose a quasi neural network to approximate the binary weight neural network. The parameter space of quasi neural network is continuous, and its gradient can be approximated using the BinaryConnect algorithm. 
We study the expressive power of the binary weight neural network by studying the RKHS of its NTK and showed its similarity with the Gaussian kernel. 
We empirically verify that quantizing the weights can reduce the generalization gap, similar to Gaussian Kernel versus the Laplacian kernel. 
This result can be easily generalized to a neural network with other weight quantization methods, i.e. using more bits. Yet there are several questions to be answered by future work:
\begin{enumerate}
    \item In this work, we only quantize the weights, while much empirical work to quantize both the weights and the activations has been done. Can we use a similar technique to study the expressive power of that neural network?
    \item We study the NTK of a two-layer neural network with one additional linear liner in front of it, and only the weights in the first layer are quantized. It remains to be answered whether a multi-layer neural network allow similar approximation, and whether using more layers can increase its expressive power.
\end{enumerate}

\section{Acknowledgement}
The authors would like to thank Chunfeng Cui for helpful discussion about the idea of quasi neural network and effort in proofreading the paper, as well as Yao Xuan for helpful discussion about kernel and RKHS.

\bibliography{ref}
\bibliographystyle{abbrvnat}

%%%%%%%%%%%%%%%%%%%%%%%%%%%%%%%%%%%%%%%%%%%%%%%%%%%%%%%%%%%%
\clearpage
\newpage
\appendix
% \renewcommand{\thesection}{\Alph{section}}
% \setcounter{section}{0}
% \renewcommand*{\theHsection}{chX.\the\value{section}}
% %\renewcommand{\thechapter}{A}
% \onecolumn
\section{BinaryConnect Algorithm}
\label{sec:binaryconnect}
In this section, we briefly review the BinaryConnect algorithm \citep{courbariaux2015binaryconnect}, which is the key algorithm we are studying in this paper.

\begin{algorithm}[H]
    \caption{Training binary weight neural network with BinaryConnect algorithm and SGD.}
    \KwData{Training data $\{\vect x_i, z_i\}$, (Randomly) initialized model parameters $\theta^{(0)}$,learning rate $\eta$.}
    \For{$t$ in $[0 \dots T]$} {
        Select a minibatch $\{\vect x_b, z_b\}$ from training data\;
        \textbf{Quantization: } $w^{(t)} \leftarrow Quantize(\theta^{(t)})$\;
        \textbf{Forward propagation:} compute $\{y_b\}$ using $\{\vect x_b\}$ and $w^{(t)}$\;
        \textbf{Backward propagation:} compute $\frac{\partial \loss(y_b, z_b)}{\partial w^{(t)}}$ using $\{y_b\}, \{z_b\}$ and $w^{(t)}$\;
        \textbf{accumulate gradients:} $\theta^{(t+1)} \leftarrow \theta^{(t)} + \eta \frac{\partial \loss(y_b, z_b)}{\partial w^{(t)}}$.
    }
\end{algorithm}

\section{Gaussian approximation in quantized neural network}

\subsection{Proof of Proposition \ref{prop:asconv}}
For the first statement, observe that fixing $k, k'$ and taking $i$ as the variable, $w_{0, ki}w_{0, k'i}$ are independent from each other. Furthermore, $\E[w_{0, ki}w_{0, k'i}] = \delta_{k, k'}$. From the strong law of large number (SLLN), the first statement is proved. The third statement can be proved similarly, observing that $\E[w_{0, ki}w_{0, k'i}\theta_{1, ij}^2] = \delta_{k, k'}\Var[\theta]$.

To prove the second statement, because geometric mean is no larger than cubic mean, 
$$
|w_{0, ki}| |w_{0, k'i}| |w_{0, k'i}| \leq \frac{1}{3}(|w_{0, ki}|^3+ |w_{0, k'i}|^3+ |w_{0, k'i}|^3),
$$
Since $w_{0, ki} \sim \cN(0, 1)$, the expectation of the right hand side equals $\sqrt{\frac{8}{\pi}}$. We apply SLLN again to finish the proof.

\subsection{Proof of Lemma \ref{lemma:linear}}
\label{sec:linearproof}
% To prove that $y_{1, j}$ converge to Gaussian distribution, We only need to prove that Lyapunov's condition is satisfied with high probability (over random initialization of $\Theta$):
% \begin{eqal*}
%     \lim_{d_1 \rightarrow \infty} \frac{1}{d_1^{3/2}} \sum_{i=1}^{d_1} \E\left[|(w_{1, ij}-\theta_{1, ij})^3 x_{1, i}^3|\middle| \Theta \right]=0
% \end{eqal*}
% Note that $\theta_{1, ij}$ and $x_{1, i}$ are independent. First remind that $|w_{1, ij}-\theta_{1, ij}| \leq 2$, so
% \begin{eqal*}
%     \E\left[|(w_{1, ij}-\theta_{1, ij})^3 x_{1, i}^3|\middle| \Theta \right] 
%     \leq 8|x_{1, i}^3|, \forall i, j
% \end{eqal*}
% Under the law of large number,
% \begin{eqal*}
%     \lim_{d_1 \rightarrow \infty} \frac{1}{d_1} \sum_{i=1}^{d_1} |x_{1, i}^3| = \E[|x_{1, i}^3|] < \infty
% \end{eqal*}
% almost surely, so on the limit $d_1 \rightarrow \infty$, Lyapunov's condition is satisfied almost surely.

We first compute the conditioned mean and variance $\nu_{1, j}$ and $\varsigma_{1, j}$. Notice that conditioned on $\Theta$, $x_{1}$ is deterministic, 
\begin{eqal*}
    \nu_{1,j} &= \E_{w_1}\left[y_{1,j} \middle| \Theta\right]\\
    &= \sqrt{\frac{c}{d_{1}}} \sum_{i=1}^{d_{1}} \E_{w_1}[w_{1,ij}] x_{1,i} + \beta b_{1,j}\\
    &= \sqrt{\frac{c}{d_{1}}} \sum_{i=1}^{d_{1}}\theta_{1, ij} x_{1,i} + \beta b_{1,j}\\
    %-----------------------------
    \varsigma_{1, j}
    &=\Var_{w_1}\left[y_{1,j} \middle| \Theta\right]\\
    &= \E_{w_1}\left[y_{1,j}^2 \middle| \Theta \right] - \E_{w_1}\left[y_{1,j} \middle| \Theta\right]^2\\
    &= \frac{c}{d_1} \sum_{i=1}^{d_1}\sum_{i'=1}^{d_{1}} \E_{w_1}\left[w_{1,ij} w_{1,i'j} \middle| \Theta\right] x_{1,i}x_{1,i'}
    + 2\beta b_{1,j} \sqrt\frac{c}{d_1}\sum_{i=1}^{d_{1}}\E\left[w_{1,ij} \middle| \Theta\right] x_{1,i} \\
    &\quad- \frac{c}{d_1} \sum_{i=1}^{d_1}\sum_{i'=1}^{d_{1}} \theta_{1,ij} \theta_{1,i'j}  x_{1,i}x_{1,i'} 
    - 2\beta b_{1,j} \sqrt\frac{c}{d_1}\sum_{i=1}^{d_{1}}\theta_{1,ij} x_{1,i}\\
    &= \frac{c}{d_{1}} \sum_{i=1}^{d_{1}} \sum_{i'=1}^{d_{1}} (\E[w_{1,ij}w_{1,i'j}|\Theta]-\theta_{1,ij}\theta_{1,i'j}) x_{1,i}^2\\
    % &=\frac{c}{d_{1}} \sum_{i=1}^{d_{1}}\sum_{i'=1}^{d_{1}} \E\left[w_{1, ij} \middle| \Theta\right]\E\left[w_{1,i'j} \middle| \Theta\right] x_{1,i}x_{1,i'}
    % -\beta b_j \sum_{i=1}^{d_1}\E\left[w_{1,ij} \middle| \Theta\right] x_{1, i} \\
    % &= \frac{c}{d_{1}}\sum_{i=1}^{d_{1}} \left(\E\left[w_{1,ij}^2 \middle| \Theta\right] - \E\left[w_{1,ij}\middle| \Theta\right]^2 \right)x_{1,i}^2 \\
    &= \frac{c}{d_{1}} \sum_{i=1}^{d_{1}}(\E[w_{1,ij}^2|\Theta]-\theta_{1,ij}^2) x_{1,i}^2\\
    &= \frac{c}{d_{1}} \sum_{i=1}^{d_{1}}(1-\theta_{1,ij}^2) x_{1,i}^2.
    \end{eqal*}
    %here we make use of the fact that $(w_{\ell,j})^2=1$.
The second line is because $\E_{w_1}[w_{1,ij}w_{1,i'j}|\Theta]=\E_{w_1}[w_{1,ij}|\Theta]\E[w_{1,i'j}|\Theta]=\theta_{1,ij}\theta_{1,i'j}$ when $i\neq i'$. 

Next, we need to prove $y_{1, i}$ converge to Gaussian distribution conditioned on $\Theta \in \cG$ by verifying Lyapunov's condition. Note that for any $j \in [d_2]$,
\begin{equation*}
    y_{1, j} = \sqrt{\frac{c}{d_1}}\sum_{i=1}^{d_1} w_{1, ij}x_{1, i} + b_{1, j}
\end{equation*}
Define $X_i = w_{1, ij}x_{1, i}$. As mentioned above, its mean and variance (conditioned on $\Theta$) is 
\begin{eqal*}
    \E_{w_1}[X_i|\Theta] &= \theta_{1, ij} x_{1, i},&
    \Var_{w_1}[X_i|\Theta] &=  \E_{w_1}[X_i^2|\Theta] -  \E_{w_1}[X_i|\Theta]^2 = (1-\theta_{1, ij}^2) x_{1, i}^2
\end{eqal*} 
Since $\Theta \in \cG$, $\forall j \in [d_2]$ for some finite $d_2$,
\begin{eqal}
    \label{eq:condvar}
    \lim_{d_1 \rightarrow \infty}\frac{1}{d_1} \sum_{i=1}^{d_1}\Var_{w_1}[X_i|\Theta] 
    &= \lim_{d_1 \rightarrow \infty}\frac{1}{d_1} \sum_{i=1}^{d_1} (1-\theta_{1, ij}^2) x_{1, i}^2\\
    &= \lim_{d_1 \rightarrow \infty}\frac{1}{dd_1} \sum_{i=1}^{d_1} (1-\theta_{1, ij}^2) \Big(\sum_{k=1}^d w_{0, ki} x_k\Big)^2\\
    &= \lim_{d_1 \rightarrow \infty}\frac{1}{dd_1} \sum_{i=1}^{d_1}\sum_{k, k'=1}^d (1-\theta_{1, ij}^2) w_{0, ki} w_{0, k'i} x_k x_{k'}\\
    &= \lim_{d_1 \rightarrow \infty}\sum_{k, k'=1}^d (1-\Var[\theta]) \delta_{k, k'} x_k x_{k'}
    =  1-\Var[\theta] 
\end{eqal}
The fourth equality comes from the definition of $\cG$, and the fifth equality is because $\|x\|_2 = 1$.
The third order absolute momentum can be bounded by 
\begin{eqal}
    \label{eq:Lyapunov}
    &\lim_{d_1 \rightarrow \infty}\frac{1}{d_1} \sum_{i=1}^{d_1}\E_{w_1}\big[|X_i - \E_{w_1}[X_i|\Theta]|^3\big|\Theta\big]\\
    &= \lim_{d_1 \rightarrow \infty}\frac{1}{d_1} \sum_{i=1}^{d_1}\E_{w_1}\big[|(w_{1, ij}-\theta_{1, ij})x_{1, i}|^3\big|\Theta\big]\\
    & \leq \lim_{d_1 \rightarrow \infty}\frac{1}{d_1} \sum_{i=1}^{d_1}\E_{w_1}\big[8|x_{1, i}|^3\big|\Theta\big]\\
    &= \lim_{d_1 \rightarrow \infty}\frac{8}{d_1} \sum_{i=1}^{d_1}\Big|\sum_{k=1}^d w_{0, ki} x_k\Big|^3\\
    & \leq \lim_{d_1 \rightarrow \infty}\frac{8}{d_1} \sum_{i=1}^{d_1}\Bigg(\sum_{k=1}^d| w_{0, ki} x_k|\Bigg)^3\\
    & = \lim_{d_1 \rightarrow \infty}\frac{8}{d_1} \sum_{i=1}^{d_1}\sum_{k, k', k''=1}^d| w_{0, ki}w_{0, k'i}w_{0, k''i}||x_kx_{k'}x_{k''}|\\
    &\leq 8 \sqrt{\frac{8}{\pi}}d^3 
\end{eqal}

The last inequality comes from the definition of ``Good Initialization'': for all $\Theta \in \cG$,  $$\lim_{d_1 \rightarrow \infty}\frac{1}{d_1}\sum_{i=1}^{d_1} w_{0, ki} w_{0, k'i} w_{0, k'i} \leq \sqrt{\frac{8}{\pi}},$$ and because $\|x\|_2=1, |x_k| \leq 1$ for all $k \in [d]$. Note that using the strong law of large number, one can prove that the third order absolute momentum converges almost surely to a constant that doesn't depend on $d$. On the other hand, we are proving a upper bound for all $\Theta \in \cG$ which is stronger than almost surely converge.

\begin{eqal*}
    &\lim_{d_1 \rightarrow\infty} \frac{\sum_{i=1}^{d_1} \E_{w_1}[|X_i - \E_{w_1}[X_i|\Theta]|^3|\Theta]}{(\sum_{i=1}^{d_1} \Var_{w_1}[X_i|\Theta])^{3/2}}\\
    &= \lim_{d_1 \rightarrow\infty} \frac{\frac{1}{d_1}\sum_{i=1}^{d_1} \E_{w_1}[|X_i - \E_{w_1}[X_i|\Theta]|^3|\Theta]}{{d_1^{1/2}}(\frac{1}{d_1}\sum_{i=1}^{d_1} \Var_{w_1}[X_i|\Theta])^{3/2}}\\
    &= \frac{ \lim_{d_1 \rightarrow\infty}\frac{1}{d_1}\sum_{i=1}^{d_1} \E_{w_1}[|X_i - \E_{w_1}[X_i|\Theta]|^3|\Theta]}{ \lim_{d_1 \rightarrow\infty}{d_1^{1/2}}(\frac{1}{d_1}\sum_{i=1}^{d_1} \Var_{w_1}[X_i|\Theta])^{3/2}}\\
    & \leq \frac{8 \sqrt{\frac{8}{\pi}}d^3 }{\lim_{d_1 \rightarrow\infty} d_1^{1/2}(1-\Var[\theta])}\\
    &= 0
\end{eqal*}
This proves that Lyapunov's condition for all ``Good Initialization'', so conditioned on $\Theta \in \cG$, $y_{1, j}$ converges to Gaussian distribution.

\subsection{Proof of Lemma \ref{lemma:relu}}
% We introduce an auxiliary variable $y \sim \cN(\nu_{2, j}, \varsigma_{2, j}^2)$.
% Using continuous mapping theorem, because $\act$ function is continuous and $y_{2, j} \rightarrow \cN(\nu_{2, j}, \varsigma_{2, j}^2)$ in distribution, $x_{3, j} = \act(y_{2, j}) \rightarrow \act(y)$ in distribution. 
%Let $y_{\ell,i} \sim \cN(\nu_{\ell,i}, \varsigma_{\ell,i})$, $x_{\ell+1,i} = ReLU(y_{\ell,i}) = \max(y_{\ell,i}, 0)$.
To compute $\E_{w_1}[x_{2,j}|\Theta]$ and $\Var_{w_1}[x_{2,j}|\Theta]$, we first compute $\E_{w_1}[\act(y_{1, j})|\Theta]$ and $\E_{w_1}[\act(y_{1, j})^2|\Theta]$. Recall $\act(x) = x \mathbf{1}(x \geq 0)$,
\begin{eqal*}
    \E_{w_1}[\act(y_{1, j})|\Theta] 
    %&= \int_{-\infty}^{\infty} \max(\hat x, 0) \frac{1}{\sqrt{2\pi}\sigma} \exp\left(-\half\frac{(\hat x-\mu)^2}{\sigma^2}\right)d\hat x\\
    &= \int_{0}^{\infty}  x\frac{1}{\sqrt{2\pi}\varsigma_{1, j}} \exp\left(-\half\frac{( x-\nu_{1,j})^2}{\varsigma_{1, j}^2}\right)d x\\
    &= \int_{-\frac{\nu_{1,j}}{\varsigma_{1, j}}}^\infty (\varsigma_{1, j} y + \nu_{1,j})\frac{1}{\sqrt{2\pi}}\exp\left(-\half y^2\right)d y\\
    &= \varsigma_{1, j} \int_{-\frac{\nu_{1,j}}{\varsigma_{1, j}}}^\infty \frac{1}{\sqrt{2\pi}}y\exp\left(-\half y^2\right)d y
    + \mu_{\ell,i}\int_{-\frac{\nu_{1,j}}{\varsigma_{1, j}}}^\infty \frac{1}{\sqrt{2\pi}}\exp\left(-\half y^2\right)d y,
\end{eqal*}
\begin{eqal*}
    %---------------------------------------
    \E_{w_1}[\act(y_{1, j})^2|\Theta] 
    &= \int_{0}^{\infty} x^2\frac{1}{\sqrt{2\pi}\varsigma_{1, j}} \exp\left(-\half\frac{(x-\nu_{1,j})^2}{\varsigma_{1, j}^2}\right)d x\\
    &= \int_{-\frac{\nu_{1,j}}{\varsigma_{1, j}}}^\infty (\varsigma_{1, j} y+\nu_{1,j})^2\frac{1}{\sqrt{2\pi}}\exp\left(-\half y^2\right)d y\\
    &= \varsigma_{1, j}^2 \int_{-\frac{\nu_{1,j}}{\varsigma_{1, j}}}^\infty  y^2 \frac{1}{\sqrt{2\pi}}\exp\left(-\half\hat y^2\right)d y + 2\varsigma_{1, j}\nu_{1,j}\int_{-\frac{\nu_{1,j}}{\varsigma_{1, j}}}^\infty  y \frac{1}{\sqrt{2\pi}}\exp\left(-\half y^2\right)d y\\
    &\quad + \nu_{1,j}^2\int_{-\frac{\nu_{1,j}}{\varsigma_{1, j}}}^\infty \frac{1}{\sqrt{2\pi}}\exp\left(-\half y^2\right)d y.
\end{eqal*}
We only need to compute 
$$
\int_{-\frac{\nu_{1,j}}{\varsigma_{1, j}}}^\infty \frac{1}{\sqrt{2\pi}}y^\alpha\exp\left(-\half y^2\right)d y.
$$
For $\alpha=0, 1, 2$. When $\alpha=0$, this is integration to Gaussian function, and it is known that there's no analytically function to express that.
For sack of simplicity, define it as $s_{1,j}$ 
\begin{equation*}
\begin{aligned}
    s_{1,j}
    &= \int_{\frac{\nu_{1,j}}{\varsigma_{1, j}}}^\infty \frac{1}{\sqrt{2\pi}}\exp\left(-\half y^2\right)d y 
    := \Phi(\frac{\nu_{1,j}}{\varsigma_{1, j}}).
\end{aligned}
%\label{eq:s}
\end{equation*}
When $\alpha=1$, this integration can be simply solved by change of the variable and we denote it as $g_{1, j}$:
\begin{equation*}
\begin{aligned}
    g_{1,j}
    &= \int_{-\frac{\nu_{1,j}}{\varsigma_{1, j}}}^\infty  y \frac{1}{\sqrt{2\pi}}\exp\left(-\half y^2\right)d y  = \sqrt{\frac{1}{2\pi}}\exp\left(-\half \left(\frac{\nu_{1,j}}{\varsigma_{1, j}}\right)^2 \right).
\end{aligned}
%\label{eq:g}
\end{equation*}
When $\alpha=2$, we can do integration by parts and express it using $s_{1,j}$ and $g_{1, j}$:
\begin{equation*}
\begin{aligned}
    &\int_{-\frac{\nu_{1,j}}{\varsigma_{1, j}}}^\infty  y^2 \frac{1}{\sqrt{2\pi}}\exp\left(-\half y^2\right)d y \\
    &\quad = -\int_{-\frac{\nu_{1,j}}{\varsigma_{1, j}}}^\infty  y \frac{1}{\sqrt{2\pi}}\exp\left(-\half y^2\right)d \frac{1}{2}y^2\\
    &\quad = \int_{-\frac{\nu_{1,j}}{\varsigma_{1, j}}}^\infty \frac{1}{\sqrt{2\pi}}\exp\left(-\half y^2\right)d y 
    - \frac{\nu_{1,j}}{\varsigma_{1, j}}\frac{1}{\sqrt{2\pi}}\exp\left(-\half \left(\frac{\nu_{1,j}}{\varsigma_{1, j}}\right)^2 \right)\\
    &\quad = s_{1,j} -\frac{\nu_{1,j}}{\varsigma_{1, j}} g_{1,j}.
\end{aligned}
%\label{eq:esq}
\end{equation*}

Using the definition of mean and variance,
\begin{equation*}
    \mu_{2,j} = \E_{w_1}[\act(y_{1, j})|\Theta], \sigma_{2,i}^2 = \E_{w_1}[\act(y_{1, j})^2|\Theta]-\E_{w_1}[\act(y_{1, j})|\Theta]^2,
\end{equation*}
we can come to the result.

\subsection{Proof of Theorem \ref{thm:var}}
\label{sec:prfvar}
In this part, we take $\Theta = \{w_0, \theta_1, w_2, b_0, b_1, b_2\}$ as the random variables and conditioned mean and variance derived above $\mu_1, \sigma_1, \nu_1, \varsigma_1$ as functions to $\Theta$.
From Eq. \eqref{eq:linear}, 
%using central limit theorem, as $d_1\rightarrow \infty$ $v_{1, j}$ converge to joint Gaussian distributed for any finite set of $\{j\}$. 
as $d_1 \rightarrow \infty$, $v_{1}$ tend to iid Gaussian processes, and there covariance converges almost surely to its expectation. We then focus on computing the expectation of covariance.
For any $j \neq j'$, we take the expectation over random initialization of $\Theta$:
\begin{eqal}
    &\E_\Theta[\nu_{1, j}\nu_{1, j'}] \\
    &= \E_\Theta\left[\frac{c}{d_1} \sum_{i=1}^{d_1}\sum_{i'=1}^{d_1}\theta_{1, ij}\theta_{1, i'j'} x_{1, i}x_{1, i'} 
    + \beta\sqrt{\frac{c}{d_1}}\sum_{i=1}^{d_1}(\theta_{1, ij}x_{1, i}b_j + \theta_{1, ij'} x_{1, i}b_{j'}) + \beta^2 b_j b_{j'} \right]\\
    %---------------------------------
    &= \frac{c}{d_1} \sum_{i=1}^{d_1}\sum_{i'=1}^{d_1}\E_\Theta[\theta_{1, ij}]\E_\Theta[\theta_{1, i'j'}] \E_\Theta[x_{1, i}x_{1, i'}] 
    + \beta^2 \E_\Theta[b_j b_{j'}] \\
    &\quad + \sqrt{\frac{c}{d_1}}\beta\sum_{i=1}^{d_1}(\E_\Theta[\theta_{1, ij}]\E_\Theta[x_{1, i}] \E_\Theta[b_j] + \E_\Theta[\theta_{1, ij'}]\E_\Theta[x_{1, i}] \E_\Theta[b_{j'}]) \\
    &= 0
\end{eqal}
which indicates that they are independent.

Computation of $\varsigma_{1, j}$ was already finished implicitly in Section \ref{sec:linearproof}. We write it explicitly here. From \eqref{eq:linear}, on the limit $d_1 \rightarrow \infty$,
\begin{eqal*}
    \varsigma_{1, j}^2 &= \frac{c}{d_1} \sum_{i=1}^{d_1} (1-\theta_{1, ij}^2) x_{1, i}^2 \\
    &= \frac{c}{dd_1} \sum_{i=1}^{d_1} (1-\theta_{1, ij}^2) \sum_{k, k'=1}^d w_{0, ki}w_{0, k'i}x_{k}x_{k'}\\
    &= \frac{c}{d}\sum_{k, k'=1}^d x_{k}x_{k'} \frac{1}{d_1}\sum_{i=1}^{d_1}w_{0, ki}w_{0, k'i}(1-\theta_{1, ij}^2)\\
    &= \frac{c}{d}\sum_{k, k'=1}^d x_{k}x_{k'} \delta_{k, k'}(1-\Var[\theta])\\
    &= \frac{c}{d}(1-\Var[\theta])\|x\|_2^2\\
    &= \frac{c}{d}(1-\Var[\theta])
\end{eqal*}
The fourth line comes from the definition of $\cG$.

% Furthermore, because $\theta_{1, ij}$ and $x_{\ell, i}$ are iid. distributed, according to the low of large number,
% \begin{equation*}
%     \varsigma_1^2 =\frac{c}{d_1} \sum_{i=1}^{d_1} (1-\theta_{1, ij}^2) x_{1, i}^2 
% \end{equation*}
% as $d_1\rightarrow \infty$, it converges to the expectation 
% \begin{eqal*}
%     \E[c(1-\theta_{1, ij}^2) x_{1, i}^2 ]&=
%     c(1-\E[\theta_{1,ij}^2])\E[{\mu_{1, i}}^2]\\
%     &= \frac{c}{d}(1-\Var[\theta])
% \end{eqal*}
% almost surely. The second equation comes from the fact
% \begin{eqal*}
%     \E[{\mu_{1, i}}^2]
%     &= \E[\frac{1}{d} \sum_{k=1}^d w_{1, ki}x_i]^2\\
%     &= \frac{1}{d}\sum_{k=1}^d\sum_{k'=1}^d \E[w_{1, ki} w_{1, k'i}] \E[x_kx_k']\\
%     &= \frac{1}{d}\sum_{k=1}^d\sum_{k'=1}^d \delta_{k, k'} \E[x_kx_k']\\
%     &= \frac{1}{d}\sum_{k=1}^d\E[x_kx_k']\\
%     &= \frac{1}{d}
% \end{eqal*}

%From Slutsky's theorem, $\frac{\nu_{\ell, j}}{\varsigma_{\ell, j}}$ converges to $\frac{\nu_{\ell, j}}{\bar\varsigma_{\ell}}$ in distribution.

%\red{Question: can we replace $\varsigma_{\ell,i}$ with its expectation in the following layer? What's the error introduced by that?}

\subsection{Derivative of activation function in quasi neural network}
\label{sec:phip}
Let 
\begin{eqal*}
\tilde\act(x) =  \tilde\varsigma_1 \phi\left( \frac{x}{\tilde\varsigma} \right) + x \Phi\left(\frac{x}{\tilde\varsigma}\right),
\end{eqal*}
Its derivative is 
\begin{eqal*}
\tilde\act'(x) &= 
\varphi'\left( \frac{x}{\tilde\varsigma} \right) +  \Phi\left(\frac{x}{\tilde\varsigma}\right) + \frac{x}{\tilde\varsigma}\Phi'\left(\frac{x}{\tilde\varsigma}\right)\\
&= -\frac{x}{\tilde\varsigma}\varphi\left( \frac{x}{\tilde\varsigma} \right) + \Phi\left(\frac{x}{\tilde\varsigma}\right) + \frac{x}{\tilde\varsigma}\varphi\left(\frac{x}{\tilde\varsigma}\right)\\
&= \Phi\left(\frac{x}{\tilde\varsigma}\right)
\end{eqal*}
The second line is because 
\begin{eqal*}
\Phi'(x) &= \varphi(x),\\
\varphi'(x) &= \frac{d}{dx} \sqrt{\frac{1}{2\pi}}\exp\left(-\half x^2 \right)\\
&=-x\sqrt{\frac{1}{2\pi}}\exp\left(-\half x^2 \right)\\
&=-x\varphi(x).
\end{eqal*}

\subsection{Proof of Theorem \ref{thm:nngrad}}

% \begin{lemma} 
% \label{lemma:nngrad} The gradient with respect to weights and bias in equation (\ref{eq:2nn}) can be computed by:
% \begin{equation}
% \begin{aligned}
%     \frac{\partial \E[y]}{\partial w_{2,ij}}
%     &= \sqrt{\frac{3}{d_2d_3}} w_{3,j}\tilde x_i\tilde\act'(\nu_j) +O\left(\frac{1}{d_2\sqrt{d_3}}\right),\\ %\tilde\sigma'(\mu_j)(1+O(\frac{1}{\sqrt{n}}))\\ % + O(\frac{1}{\sqrt{m}d})\\
%     \frac{\partial \E[y]}{\partial b_{2,j}}
%     &= \beta\sqrt{\frac{1}{d_3}} w_{3,j} \tilde\act'(\mu_j),% + O(\frac{1}{\sqrt{md}})\\
%     \quad
%     \frac{\partial \E[y]}{\partial w_{3,j}} 
%     = \sqrt{\frac{1}{d_3}} \tilde\act(\nu_{2,j}).
% \end{aligned}
% \label{eq:dedp}
% \end{equation}
% \end{lemma}

To make the proof more general, we make $\varsigma_{1, j}$ a parameter of the activation function in quasi neural network as $\tilde\act(\cdot; \varsigma_{1, j})$.
To get the derivative with respect to $\theta_{1, ij}$, we first get the derivative with respect to $\nu_{1,j}$. 
\begin{eqal*}
    \frac{\partial \bar y}{\partial \nu_{1, j}}
    &= \frac{\partial \bar y}{\partial \mu_{2, j}}
    \frac{\partial \mu_{2, j}}{\partial \nu_{1, j}}
    =\sqrt\frac{1}{d_2} w_{2, j} \tilde\act'(\nu_{1, j}; \varsigma_{1, j})
\end{eqal*}
%Note $\tilde\act'(x)=\Phi(x), \phi'(x) = -x\phi(x), \Phi'(x) = \phi(x)$ as defined in \eqref{eq:phi},
then apply chain rule:
\begin{align}
    \frac{\partial \bar y}{\partial w_{2,j}}
    &= \sqrt{\frac{c}{d_2}} \mu_{2,j},
    %= \sqrt{\frac{c}{d_3}} \tilde\act(\nu_{1, j}),\\
    \label{eq:dedppva}\\
    %-------------------
    \frac{\partial \bar y}{\partial b_{1,j}}
    &= \frac{\partial \bar y}{\partial \nu_{1,j}}\frac{\partial \nu_{1,j}}{\partial b_{1,j}} 
    = \sqrt{\frac{c}{d_2}} \beta w_{2,j} \tilde\act'(\nu_{1, j}; \varsigma_{1, j}),
    \label{eq:dedppvb}\\
    %-------------------
    \frac{\partial \bar y}{\partial \theta_{1,ij}}
    &= \frac{\partial\bar y}{\partial \nu_{1,j}}\frac{\partial \nu_{1,j}}{\partial \theta_{1, ij}} 
    = \sqrt{\frac{c}{d_1d_2}} w_{2,j}x_{1,i} \tilde\act'(\nu_{1,j}; \varsigma_{1, j}).
    \label{eq:dedppv}
\end{align}

On the other hand, 
%for any set of fixed $w_{2, ij} \in \{-1, 1\}$,
let's first write down the gradient with respect to weights $w_{ij}$ in quantized neural network and take their expectation conditioned on $\Theta$:
\begin{align}
    \E_{w_1}\left[\frac{\partial y}{\partial w_{2,j}}\middle|\Theta\right]
    &= \sqrt{\frac{c}{d_2}}\E_{w_1}[x_{2,j}|\Theta], 
    \label{eq:edppva}\\
    %-------------------
    \E_{w_1}\left[\frac{ \partial y}{\partial b_{2,j}}\middle|\Theta\right]
    &=  \E_{w_1}\left[\frac{\partial \bar y}{\partial y_{1,j}}\frac{\partial y_{1,j}}{\partial b_{2,j}} \middle|\Theta\right]
    = \sqrt{\frac{c}{d_2}} \beta w_{2,j} \E_{w_1}[\act'(y_{1, j})|\Theta],
    \label{eq:edppvb}\\
    %------------------------
     \E_{w_1}\left[\frac{\partial y}{\partial w_{1,ij}} \middle|\Theta\right]
    &=  \E_{w_1}\left[\frac{\partial y}{\partial y_{1,j}}\frac{\partial y_{1,j}}{\partial w_{1,ij}}\middle|\Theta\right]
    = \sqrt{\frac{c}{d_1 d_2}}w_{2,j} x_{1,i} \E_{w_1}[\act'(y_{1,j})|\Theta],
    \label{eq:edydw}
\end{align}

Comparing \eqref{eq:dedppva} with \eqref{eq:edppva}, one can easily find they are equal since by definition $\mu_{2,j}=\E_{w_1}[x_{2,j}|\Theta]$. On the other hand, from \eqref{eq:quasiact}, one can tell using continuous mapping theorem that
\begin{eqal*}
    \tilde\act'(\nu_{1, j}; \varsigma_{1, j}) 
    =\Phi\left(\frac{\nu_{1, j}}{\varsigma_{1, j}}\right) = P[y_{1, j} \geq 0] = \E_{w_1}[\act'(y_{1, j})|\Theta],
\end{eqal*}
which shows \eqref{eq:dedppvb} equals \eqref{eq:edppvb} and \eqref{eq:dedppv} equals \eqref{eq:edydw}.

% Finally, by definition, for $x\sim \cN(\mu, \epsilon^2)$,
% \begin{eqal*}
%     \E[\act'(x)] = P(x \geq 0) = \tilde\act(\mu).
% \end{eqal*}
% As $y_{1, j}$ converges to Gaussian distribution  $\cN(\mu_{1, j}, \varsigma_{1, j}^2)$, 
%the proof is complete using continuous mapping theorem.
%Note that $\E[x_{3, j}|\Theta]=\mu_{3, j}$, $\E[\act'(x_{2, j})|\Theta]=P[x_{2, j} \geq 0 | \Theta]\rightarrow \Phi(\nu_{2, j}/\tilde\varsigma_{2} )=\tilde\act'(\nu_{2, j})$ as $d_1, d_2, d_3 \rightarrow \infty$, which finishes the proof.

\subsection{Proof of Theorem \ref{thm:lossgrad}}
% \begin{lemma}
Observe that conditioned on $\Theta$, $y_{1, j}$ depends only on $\{w_{1, ij},  i \in [d_1]\}$, and that $\{w_{1, ij},  i \in [d_1]\} \cap \{w_{1, ij'},  i \in [d_1]\} = \emptyset$ for $j \neq j'$. Because of that,
$y_{1, j}$ are independent of each other. Similarly, $x_{2, j}$ are independent of each other conditioned on $\Theta$. 
% \end{lemma}
For MSE loss,
\begin{eqal*}
    \loss(y) = \half (y-z)^2, 
    \frac{dl(y)}{dy} = y - z,
\end{eqal*}
% where $z$ is the ground-truth label. 
% We have 
% $$
% \frac{\partial loss(\bar y)}{\partial \bar y} = \bar y - z,  \E\left[\frac{\partial loss(y)}{\partial y}\right] = \E[y] - z.
% $$
According to the chain rule
\begin{eqal}
    \frac{\partial \loss(\bar y)}{\partial \theta} = 
    \frac{\partial \loss(\bar y)}{\partial \bar y}
    \frac{\partial\bar y}{\partial \theta}  = (\bar y-z) \frac{\partial\bar y}{\partial \theta},
    \label{eq:dledt}
\end{eqal}
for any $\theta \in \{\theta_{1, ij}, b_{1, j}, w_{2, j}\}$, which leads to
\begin{align}
    \frac{\partial \loss(\bar y)}{\partial w_{2,j}}
    &= \sqrt{\frac{c}{d_2}} (\bar y - z)\mu_{2,j},
    %= \sqrt{\frac{c}{d_3}} \tilde\act(\nu_{1, j}),\\
    \label{eq:dldppva}\\
    %-------------------
    \frac{\partial \loss(\bar y)}{\partial b_{1,j}}
    &= \frac{\partial \bar y}{\partial \nu_{1,j}}\frac{\partial \nu_{1,j}}{\partial b_{1,j}} 
    = \sqrt{\frac{c}{d_2}} \beta w_{2,j} (\bar y - z)\tilde\act'(\nu_{1, j}; \varsigma_{1, j}),
    \label{eq:dldppvb}\\
    %-------------------
    \frac{\partial \loss(\bar y)}{\partial \theta_{1,ij}}
    &= \frac{\partial\bar y}{\partial \nu_{1,j}}\frac{\partial \nu_{1,j}}{\partial \theta_{1, ij}} 
    = \sqrt{\frac{c}{d_1d_2}} w_{2,j}x_{1,i} (\bar y - z)\tilde\act'(\nu_{1,j}; \varsigma_{1, j}).
    \label{eq:dldppv}
\end{align}

On the other hand, in the original binary weight neural network,
\begin{align}
    \E_{w_1}\left[ \frac{\partial \loss(y)}{\partial w_{2,j}} \middle|\Theta\right]
    &=\sqrt{\frac{c}{d_2}}\E_{w_1}\left[ (y-z)x_{2,j}|\Theta \right], \label{eq:edldw2}\\
    %---------------------------
    \E_{w_1}\left[\frac{ \partial \loss(y)}{\partial b_{1,j}}\middle|\Theta\right]
    &= \sqrt{\frac{c}{d_2}} \beta w_{2,j} \E_{w_1}[(y-z)\act'(y_{1, j})|\Theta], \label{eq:edldb1}\\
    %------------------------
    \E_{w_1}\left[\frac{\partial \loss(y)}{\partial w_{1,ij}} \middle|\Theta\right]
    &= \sqrt{\frac{c}{d_1 d_2}}w_{2,j} x_{1,i} \E_{w_1}[(y-z)\act'(y_{1,j})|\Theta],
    \label{eq:edldw1}
\end{align}
Note that $y$ is not independent form $x_{2, j}$ or $\act'(y_{1,j})$, which is the main challenge of the proof. To deal with this problem, we bound the difference between \eqref{eq:dldppva}-\eqref{eq:dldppv} and \eqref{eq:edldw2}-\eqref{eq:edldw1}, which requires bounding their covariance.
\begin{eqal}
    \E_{w_1}[x_{2,j}y|\Theta] &= \sqrt\frac{c}{d_2}\E_{w_1}\left[ x_{2,j}\sum_{j=1}^{d_2} w_{2,j} x_{2, j}\middle|\Theta\right]\\
    &= \sqrt\frac{c}{d_2}\left(\E_{w_1}\left[ x_{2,j}^2 w_{2,j}\middle|\Theta\right]+\sum_{j' \neq j}\E_{w_1}\left[ x_{2,j} x_{2, j'} w_{2,j'}\middle|\Theta\right] \right)\\
    &= \sqrt\frac{c}{d_2}\left( (\mu_{2, j}^2 + \sigma_{2,j}^2) w_{2,j}+\sum_{j' \neq j} \mu_{2,j} \mu_{2, j'} w_{2,j'} \right)\\
    &= \sqrt\frac{c}{d_2}\left( \sigma_{2,j}^2 w_{2,j}+\sum_{j'=1}^{d_2} \mu_{2,j} \mu_{2, j'} w_{2,j'} \right)\\
\end{eqal}
The second term equals $\E_{w_1}[x_{2,j}|\Theta]\E_{w_1}[y|\Theta]$ and the first term converges to 0 when $d_2 \rightarrow \infty$. Taking it into \eqref{eq:dldppva} and \eqref{eq:edldw2}, we can see that these two terms are equal on the limit $d_1 \rightarrow \infty$.
%$\E\left[ \frac{\partial loss(y)}{\partial w_{2,j}} \middle|\Theta\right]$ converges to $\frac{\partial loss(\bar y)}{\partial w_{2, j}}$ as $d_2 \rightarrow \infty$.

Similarly, 
\begin{eqal}
    &\E_{w_1}[\act'(y_{1,j})y|\Theta] \\
    %--------------------------------------
    &= \E_{w_1}\left[\sqrt\frac{c}{d_2} \act'(y_{1,j})\sum_{j=1}^{d_2} w_{2,j} \act(y_{1, j})\middle|\Theta\right]\\
    %---------------------------------------
    &= \sqrt\frac{c}{d_2}\left(\E_{w_1}\left[ \act(y_{1, j})\act'(y_{1,j}) w_{2,j}\middle|\Theta\right]+\sum_{j' \neq j}\E_{w_1}\left[ \act(y_{1, j'})\act'(y_{1,j}) w_{2,j'}\middle|\Theta\right] \right)\\
    %---------------------------------------
    &= \sqrt\frac{c}{d_2}\left( \E_{w_1}\left[ \act(y_{1, j})w_{2,j}\middle|\Theta\right] w_{2,j}
    +\sum_{j' \neq j}\E_{w_1}[\act'(y_{1,j})|\Theta]\mu_{2, j'} w_{2,j'} \right)\\
    %---------------------------------------
    &= \sqrt\frac{c}{d_2}\left( (1-\E_{w_1}[\act'(y_{1,j})|\Theta])\mu_{2,j} w_{2,j} +\sum_{j'=1}^{d_2} \E_{w_1}[\act'(y_{1,j})|\Theta] \mu_{2,j'} w_{2,j'} \right)\\
\end{eqal}
The second term equals $\E[\act'(y_{2,j})|\Theta]\E[y|\Theta]$ and the first term converges to 0 when $d_3 \rightarrow \infty$. Taking it into \eqref{eq:dldppvb}\eqref{eq:dldppv}\eqref{eq:edldb1}\eqref{eq:edldw1} finishes the proof.

\subsection{Proof of \autoref{thm:lazy1}}
\label{sec:lazyproof1}
In this part, we denote $\dot a := \frac{\partial a}{\partial t}$ for $a \in \{w_{\ell}, \theta_{\ell}, b_{\ell}\}$, and express each time-depent variable as a function of time $t$. 
We define an inner product under the distribution of training dataset 
$$
\langle \vect a, \vect b\rangle_{in} = \E_{in}[ {a}(x){b}(x)],
$$ 
and the corresponding norm $$\|\vect{a}\|_{in} = \sqrt{\langle \vect{a}, \vect{a}\rangle_{in}} = \sqrt{\E_{in}[{a}(x)^2]}.
$$ 
%if $a(x)$ is a scalar, and $\|\vect{a}\|_{in} = \sqrt{\E_{in}[\|\vect{a}(x)\|^2]}$ if $\vect{a}(x)$ is a vector. 
If $\vect{a}(x)$ is a vector, $\|\vect{a}\|_{in} := \sqrt{\E_{in}[\|\vect{a}(x)\|^2]}$.
Note this inner product and norm define a Hilbert space (not to be confused with the RKHS induced by a kernel), so by Cauchy-Schwarz inequality,
\begin{eqal*}
    |\langle \vect a, \vect b\rangle_{in}| \leq \|\vect{a}\|_{in} \|\vect{b}\|_{in}, \forall \vect a, \vect b.
\end{eqal*}

As is shown in \ref{sec:grad}, on the limit $d_1, \rightarrow \infty$, the dynamics of training this neural network using gradient descent can be written as:
\begin{eqal*}
    \dot w_{2,j}(t)
    &=\sqrt{\frac{c}{d_2}}\E_{in}[(\bar y(t) - z)\mu_{2, j}(t)] \\
    %---------------------------
    \dot b_{1,j}(t)
    &= \sqrt{\frac{c}{d_2}}\E_{in}[ \beta w_{2,j}(t) (\bar y(t) - z)\tilde \act'(\nu_{1, j}(t), \varsigma_{1, j}(t))],\\
    %------------------------
    \dot \theta_{1,ij}(t)
    &= \sqrt{\frac{c}{d_1 d_2}} \E_{in}[ w_{2,j}(t) x_{1,i} (\bar y(t) - z)\tilde\act'(\nu_{1,j}(t); \varsigma_{1, j}(t))]
\end{eqal*} 
where dot denotes the derivative with respect to $t$.
%Here we slightly abuse the notation $\tilde\phi(\cdot)$, which depends explicitly on $\varsigma_{1, j}(t)$ in this equation. This term will be eliminated later.
Note the activation function $\tilde\act(\cdot; \varsigma_{1, j}(t))$ depends on $\varsigma_{1, j}$, which makes it time dependent. One can further write down the dynamics of $\nu_{1, j}(t)$ as
\begin{eqal*}
    \dot \nu_{1, j}(t) &= \sqrt{\frac{1}{d_1}} \sum_{i=1}^{d_1} \dot \theta_{1,ij}(t)x_{1,i}(t) + \dot b_{1, j}(t)
    %&= \sqrt{\frac{c}{d_1^2d_2}} w_{2,j}(t) \sum_{i=1}^{d_1} x_{1, i} \E_{(x', z') \sim in} \Big[ x'_{1,i} (\bar y'(t) - z')\tilde\act'(\nu_{1,j}')\Big] \\
    %&\quad +  \sqrt{\frac{c}{d_2}}\E_{{x', z'}\sim in}[ \beta w_{2,j} (\bar y'(t) - z')\tilde \act'(\nu'_{1, j})].
\end{eqal*}
Rewrite these two differential equations in matrix form:
\begin{eqal*}
    \dot {\vect w}_{2}(t)
    &=\sqrt{\frac{c}{d_2}} \E_{in}[(\bar y(t) - z)\boldsymbol\mu_{2}(t)]\\
    %-----------------------------------
    \dot {\boldsymbol b}_{1}(t)
    &= \beta \sqrt{\frac{c}{d_2}}\E_{in} [(\bar y(t) - z)(\tilde \act'(\boldsymbol\nu_{1}(t)) \circ \boldsymbol w_{2}(t))],\\
    %------------------------------------
    \dot {\boldsymbol\theta_{1}}(t)
    &= \sqrt{\frac{c}{d_1 d_2}} \E_{in} \big[(\bar y(t) - z) \boldsymbol x_{1} \otimes \big(\tilde\act'(\boldsymbol\nu_{1}(t))\circ\boldsymbol w_{2}(t)\big) \big],\\
    %------------------------------------
    \dot {\boldsymbol \nu}_{1}(t) 
    %&= \sqrt{\frac{c}{d_1^2d_2}} \E_{in}[(\bar y(t) - z)  \boldsymbol x_{1} \otimes (\tilde\act'(\boldsymbol\nu_{1}(t))\circ\boldsymbol w_{2}(t)) ]  \vect x_{1} \\
    %&\quad + \beta \sqrt{\frac{c}{d_2}}\E_{in}[ (\bar y(t) - z) (\tilde \act'(\boldsymbol\nu_{1}(t)) \circ \boldsymbol w_{2}(t)) ]
    &= \sqrt{\frac{1}{d_1}} \dot{\vect\theta}_1\vect x_1 + \dot{\vect b}_1
\end{eqal*}
where $\circ$ denotes elementwise product and $\otimes$ denotes outer product. Here we slightly abuse the notation $\tilde\act(\cdot)$, which represents elementwise operation when applied to a vector.
Their norm are bounded by
\begin{eqal}
    \frac{\partial}{\partial t}\| {\boldsymbol w}_{2}(t) - {\boldsymbol w}_{2}(0)\|
    &\leq \sqrt{\frac{c}{d_2}} \E_{in}[(\bar y(t) - z)\|\vect\mu_{2}(t)\|]
    = \sqrt{\frac{c}{d_2}} \langle\bar y(t) - z, \vect\mu_{2}(t)\rangle_{in}\\
    &\leq \sqrt{\frac{c}{d_2}} \|\bar y(t) - z\|_{in} \|\boldsymbol\mu_{2}(t)\|_{in} 
    \leq \sqrt{\frac{c}{d_2}} \|\bar y(t) - z\|_{in} \|\boldsymbol\nu_{1}(t)\|_{in} 
    \label{eq:pde:norm1}
\end{eqal}
\begin{eqal}
    %-----------------------------------
    \frac{\partial}{\partial t}\| {\boldsymbol b}_{1}(t) - {\boldsymbol b}_{1}(0)\|
    &\leq \beta \sqrt{\frac{c}{d_2}}\E_{in} [(\bar y(t) - z)\| \tilde \act'(\boldsymbol\nu_{1}(t)) \circ \boldsymbol w_{2}(t)\| ]\\
    &\leq \beta \sqrt{\frac{c}{d_2}}\E_{in} [(\bar y(t) - z)\| \boldsymbol w_{2}(t)\|]
    %&\leq \beta \sqrt{\frac{c}{d_2}} \langle(\bar y(t) - z)\| \boldsymbol w_{2}(t)\|\rangle_{in}
    = \beta \sqrt{\frac{c}{d_2}}\| \bar y(t) - z\|_{in} \| \boldsymbol w_{2}(t) \|,\\
    %------------------------------------
\end{eqal}
\begin{eqal}
    \frac{\partial}{\partial t}\| {\boldsymbol\theta_{1}}(t) - {\boldsymbol\theta_{1}}(0)\|_{F}
    &\leq \sqrt{\frac{c}{d_1 d_2}} \E_{in} \big[(\bar y(t) - z) \|\boldsymbol x_{1} \otimes \big(\tilde\act'(\boldsymbol\nu_{1}(t))\circ\boldsymbol w_{2}(t)\big)\|_F\big]\\
    &\leq \sqrt{\frac{c}{d_1 d_2}} \E_{in} \big[(\bar y(t) - z) \|\boldsymbol x_{1}\|\|\boldsymbol w_{2}(t)\|\big]\\
    &\leq \sqrt{\frac{c}{d_1 d_2}} \|\bar y(t) - z\|_{in} \|\boldsymbol x_{1}\|_{in} \|\boldsymbol w_{2}(t) \|\\
    &= \sqrt{\frac{c}{d_2}} \|\bar y(t) - z\|_{in}\|\boldsymbol w_{2}(t) \|,\\
    %------------------------------------
\end{eqal}
\begin{eqal}
    \forall \vect x_1, \quad
    \frac{\partial}{\partial t}\| {\boldsymbol \nu}_{1}(t) - {\boldsymbol \nu}_{1}(0) \|
    &\leq \int_{t=0}^T \Big\|\frac{\partial {\boldsymbol \nu}_{1}(t)}{\partial t}\Big\|dt\\
    &\leq \sqrt{\frac{1}{d_1}}\frac{\partial}{\partial t}\|\vect\theta_{1}(t) - \vect\theta_{1}(0)\|_{op} \|\vect x_1\| + \frac{\partial}{\partial t}\| {\boldsymbol b}_{1}(t) - {\boldsymbol b}_{1}(0)\|\\
    &\leq \sqrt{\frac{c}{d_1^2d_2}} \|\bar y(t) - z\|_{in} \|  \boldsymbol w_{2}(t)\|  \|\boldsymbol x_{1}\|_{in}\|\boldsymbol x_{1}\| \\
    &\qquad + \beta \sqrt{\frac{c}{d_2}}\| \bar y(t) - z \|_{in} \|\boldsymbol w_{2}(t)\|\\
    &= (1+\beta)\sqrt{\frac{c}{d_2}} \|\bar y(t) - z \|_{in}  \|\vect w_{2}(t)\|, \\
    %---------------------------------
    \frac{\partial}{\partial t}\| {\boldsymbol \nu}_{1}(t) - {\boldsymbol \nu}_{1}(0) \|_{in} 
    &\leq (1+\beta)\sqrt{\frac{c}{d_2}} \|\bar y(t) - z \|_{in}  \|\vect w_{2}(t)\|. 
    \label{eq:pde:norm5}
\end{eqal}
Here we make use of the fact that $\tilde\phi'(x) \leq 1$, $\tilde\phi(x) \leq x$ regardless of the value of $\varsigma_{1, j}(t)$, that $
\lim_{d_1 \rightarrow \infty}\|\vect x_1\|_{in}/\sqrt{d_1}=1$ as long as $\Theta \in \cG$, and that $w_0$ is not updated during training. In the last equation, we make use of $\dot{\vect\theta}_1 = \frac{\partial}{\partial t}(\vect\theta_{1}(t) - \vect\theta_{1}(0)), \dot{\vect b}_1 = \frac{\partial}{\partial t}(\vect b_{1}(t) - \vect b_{1}(0)).$

Define $A(t) = \sqrt\frac{c}{d_2} \sqrt{1+\beta}(\|{\boldsymbol w}_{2}(t) - {\boldsymbol w}_{2}(0)\| + \|{\boldsymbol w}_{2}(0)\|) + \sqrt\frac{c}{d_2}(\| {\boldsymbol \nu}_{1}(t) - {\boldsymbol \nu}_{1}(0) \|_{in} + \|{\boldsymbol \nu}_{1}(0) \|_{in})$, then 
\begin{eqal*}
    \dot A(t) &
    \leq \sqrt{1+\beta}\sqrt{\frac{c}{d_2}} \|\bar y(t) - z\|_{in} \|\boldsymbol\nu_{1}(t)\|_{in} 
    + (1+\beta)\sqrt{\frac{c}{d_2}} \|\bar y(t) - z\|_{in} \|  \boldsymbol w_{2}(t)\|\\
    &\leq \sqrt{1+\beta} A(t) 
\end{eqal*}

Observe that $A (0)$ is stochastically bounded. %as we take the sequential limit $d_2\rightarrow \infty$ as in the statement of the lemma. 
%In this limit, we indeed have tha$\boldsymbol\theta_2(t)$ and $\boldsymbol\nu_2(t)$ are convergent to $\boldsymbol\theta_2(0)$ and $\boldsymbol\nu_2(0)$. 
Using Gr\"{o}nwall's Lemma, for any finite $T$:
\begin{eqal*}
    A(T)\leq A(0)\exp\Big(\int_{t=0}^T \sqrt{1+\beta} dt\Big)
    = A(0)\exp(\sqrt{1+\beta} T)
\end{eqal*}
so $A(T)$ is stochastically bounded for all finite $T$ as $d_2\rightarrow \infty$. % which in turn indicates that $\sqrt\frac{c}{d_2}(\|\vect w_2(T) - \vect w_2(0)\|)$ is stochastically bounded.  
%which indicates that as $d_2 \rightarrow \infty$, $A(T) - A(0) \rightarrow 0$ so $\|{\boldsymbol\theta}_{2}(T) - {\boldsymbol\theta}_{2}(0)\| \rightarrow 0$, $\| {\boldsymbol \nu}_{1}(t) - {\boldsymbol \nu}_{1}(0) \|_{in} \rightarrow 0$. This shows that $\boldsymbol \theta_2(T)$ is stochastically bounded for all finite $T$. Taking this into \eqref{eq:pde:norm} concludes that for all the model parameters $\Theta \in \{\boldsymbol\theta_1, \boldsymbol\theta_2, \boldsymbol b_1\}$, as $d_2\rightarrow \infty$, $\|\Theta(T) - \Theta(0)\| \rightarrow 0$. 
Furthermore,
\begin{eqal*}
    %\sqrt\frac{c}{d_2} \max_{j \in [d_2]} |w_{2, j}(T)| \leq 
    \sqrt\frac{c}{d_2} \|\vect w_{2}(T)\| 
    \leq \sqrt\frac{c}{d_2} (\|\vect w_{2}(T) - \vect w_{2}(0)\| + \|\vect w_{2}(0)\| )
\end{eqal*}
which is also stochastically bounded. Integrating \eqref{eq:pde:norm1}-\eqref{eq:pde:norm5} from 0 to $T$ finishes the proof.

\subsection{Proof of \autoref{thm:lazy2}}
\label{sec:lazyproof2}
From \eqref{eq:linear}, it's easy to get the dynamics of $\varsigma_{1}$:
\begin{eqal*}
    \frac{\partial \varsigma_{1,j}^2(t)}{\partial t}
    &= -\frac{2c}{d_1} \sum_{i=1}^{d_1}\theta_{1, ij}(t)\dot \theta_{1, ij}(t) x_{1, i}^2 \\
    %----------------------------
    |\varsigma_{1, j}^2(T) - \varsigma_{1, j}^2(0)|
    &\leq \frac{2c}{d_1} \sum_{i=1}^{d_1} x_{1, i}^2 \int_{t=0}^T| \theta_{1, ij}(t)| |\dot \theta_{1, ij}(t)|  dt\\ 
    &\leq \frac{2c}{d_1} \sum_{i=1}^{d_1} x_{1, i}^2 \int_{t=0}^T|\dot \theta_{1, ij}(t)| dt \\
    %&\leq \frac{2c}{d_1} \sum_{i=1}^{d_1} x_{1, i}^2 \max_{i \in [d_1]}\int_{t=0}^T|\dot \theta_{1, ij}(t)| dt \\
    &\leq \frac{2c}{d_1}\sqrt{\frac{c}{d_1 d_2}} \sum_{i=1}^{d_1} x_{1, i}^2 \int_{t=0}^T \E_{in}\big| w_{2,j}(t) x_{1,i} (\bar y(t) - z)\tilde\act'(\nu_{1,j}(t); \varsigma_{1, j}(t))\big| dt \\
    &\leq \frac{2c}{d_1}\sqrt{\frac{c}{d_1 d_2}} \sum_{i=1}^{d_1} x_{1, i}^2  \int_{t=0}^T |w_{2,j}(t)| \E_{in}\big|  x_{1,i} (\bar y(t) - z)\big| dt \\
    &\leq \frac{2c}{d_1}\sqrt{\frac{c}{d_1 d_2}} \sum_{i=1}^{d_1} x_{1, i}^2 \int_{t=0}^T |w_{2,j}(t)| \|x_{1,i}\|_{in} \|\bar y(t) - z\|_{in} dt \\
    &\leq \frac{2c}{d_1^{3/2}} \sum_{i=1}^{d_1} x_{1, i}^2 \|x_{1,i}\|_{in} \int_{t=0}^T C(t) \|\bar y(t) - z\|_{in} dt \\
    &\leq \frac{2c}{d_1^{3/2}} \sum_{i=1}^{d_1} x_{1, i}^2 \|x_{1,i}\|_{in} \max_{t\in[0, T]} C(t) \int_{t=0}^T  \|\bar y(t) - z\|_{in} dt\quad a.s.\\
    % &\leq \frac{2c}{d_1} \sum_{i=1}^{d_1} x_{1, i}^2 \sqrt{\frac{c}{d_1 d_2}} \big|\E_{in}[ w_{2,j}(t) x_{1,i} (\bar y(t) - z)\tilde\act'(\nu_{1,j}(t); \varsigma_{1, j}(t))]\big|\\
    % &\leq \frac{2c}{d_1} \sum_{i=1}^{d_1} x_{1, i}^2 \sqrt{\frac{c}{d_1 d_2}} \E_{in}\big| w_{2,j}(t) x_{1,i} (\bar y(t) - z)\tilde\act'(\nu_{1,j}(t); \varsigma_{1, j}(t))\big|\\
    % &\leq \frac{2c}{d_1} \sum_{i=1}^{d_1} x_{1, i}^2 \sqrt{\frac{c}{d_1 d_2}} \E_{in}\big| w_{2,j}(t) x_{1,i} (\bar y(t) - z)\big|\\
\end{eqal*}
Here we assume that $\sqrt\frac{c}{d_2}\|\vect w_{2}(t)\|$ is stochastically bounded by $C(t)$. Since $C(t)$ is finite for all $t \in [0, T]$, it's easy to check the term after $\max$ operator is stochastically bounded. The remaining task is to bound term before $\max$ operator. From standard Gaussian process analysis, $x_{1, i}$ satisfy Gaussian distribution. From the law of large number (LLN), as $d_1 \rightarrow \infty$, 
$$
    \frac{1}{d_1}\sum_{i=1}^{d_1} x_{1, i}^2 \|x_{1,i}\|_{in} = \E[x_{1, i}^2 \|x_{1,i}\|_{in}]
$$
almost surely, where the expectation is taken over $w_1$, and this limit is also bounded. Because of that, as $d_1, d_2 \rightarrow \infty$, the difference $|\varsigma_{1, j}^2(T) - \varsigma_{1, j}^2(0)|$ converges to 0 at rate $\frac{1}{\sqrt{d_2}}$.

Notice that the proof of Lyapunov's condition \eqref{eq:Lyapunov} doesn't depend on time $T$ from the third line. Since $\varsigma_{1, j}(T)$ stochastically converges to $\varsigma_{1, j}(0)$ for all finite $T$, Lyapunov's condition holds for all $T$ thus $x_{2, j}$ always converges to Gaussian distribution conditioned on model parameter.

\section{NTK of neural networks with quantized weights}
\subsection{Spherical harmonics}
\label{sec:sh}
This subsection briefly reviews the relevant concepts and properties of spherical harmonics. Most part of this subsection comes from \citet[appendix Section D.1.]{bach2017breaking}  and \citet[appendix Section C.1.]{bietti2019inductive} 

According to Mercer's theorem, any positive definite kernel can be decomposed as 
\begin{equation*}
    \cK(x, x') = \sum_{i} \lambda_i \Phi(x) \Phi(x'),
\end{equation*}
where $\Phi(\cdot)$ is called the feature map. Furthermore, any zonal kernel on the unit sphere, i.e., $\cK(x, x') = \cK(x^Tx')$ for any $x, x' \in \R^d, \|x\|_2=\|x'\|_2=1$, including exponential kernels and NTK, can be decomposed using spherical harmonics (equation (\ref{eq:spdecom})):
\begin{equation*}
    \cK(x, x')=\sum_{k=1}^\infty \lambda_k \sum_{j=1}^{N(d, k)} Y_{k, j}(x) Y_{k, j}(x').
\end{equation*}

\textbf{Legendre polynomial.} We have the additional formula
\begin{equation*}
    \sum_{j=1}^{N(d, k)} Y_{k, j}(x) Y_{k, j}(x') = N(d, k) P_k(x^Tx'),
\end{equation*}
where 
\begin{equation*}
    N(d, k)=\frac{(2k+d-2)(k+d-3)!}{k!(d-2)!}.
\end{equation*}
The polynomial $P_k$ is the $k$-th Legendre polynomial in $d$ dimension, also known as Gegenbauer polynomials:
\begin{equation*}
    P_k(t) = \left(-\half\right)^k \frac{\Gamma\left( \frac{d-1}{2} \right)}{\Gamma\left( k+\frac{d-1}{2} \right)}
    (1-t^2)^{(3-d)/2} 
    \left(\frac{d}{dt}\right)^k (1-t^2)^{k+(d-3)/2}.
\end{equation*}
It is even (resp. odd) when $k$ is odd (reps. even). Furthermore, they have the orthogonal property
\begin{equation*}
    \int_{-1}^1 P_k(t) P_j(t) (1-t^2)^{(d-3)/2}dt = \delta_{ij}\frac{w_{d-1}}{w_{d-2}}\frac{1}{N(d, k)} ,
\end{equation*}
where 
\begin{equation*}
    w_{d-1} = \frac{2\pi^{d-2}}{\Gamma(d/2)} 
\end{equation*}
denotes the surface of sphere $\S^{d-1}$ in $d$ dimension, and this leads to the integration property
\begin{equation*}
    \int P_j(\langle w, x \rangle) P_k(\langle w, x \rangle) d\tau(w) = \frac{\delta_{jk}}{N(p, k)} P_k(\langle x, y \rangle)
\end{equation*}
for any $x, y \in \S^{d-1}$. $\tau(w)$ is the uniform measure on the sphere.

\subsection{NTK of quasi neural network}

We start the proof of the Theorem \ref{thm:exp} by the following lemmas:
\begin{lemma}
\label{lemma:qntk}
The NTK of a binary weight neural network can be simplified as 
\begin{eqal}
    &\cK(x, x') = \left(\frac{c}{d}\langle x, x' \rangle + \beta^2\right) \Sigma^{(0)} + \Sigma^{(1)},\\
    &\Sigma^{(0)} = \E\left[\tilde\act'\left(\mu \right) \tilde\act'\left(\mu'\right)\right],\quad 
    \Sigma^{(1)} = \E\left[\tilde\act\left(\mu \right) \tilde\act\left(\mu'\right)\right],
\label{eq:recntk}
\end{eqal}
where $[\mu, \mu'] \sim \cN(0, \Sigma)$,
\begin{eqal*}
    \Sigma = \E[x_{1,i} x'_{1,i}]
    = \frac{c}{d}\Var[\theta]
    \left[
        \begin{matrix}
        1 & x^Tx'\\
        x^Tx' & 1
        \end{matrix}
    \right]
\end{eqal*}
are the pre-activation of the second layer. 

\end{lemma}
\begin{proof}
\begin{eqal*}
    \cK(x, x') &= 
    \sum_{i=1, j=1}^{d_1, d_2} \frac{\partial \bar y}{\partial \theta_{1, ij}}  \frac{\partial \bar y'}{\partial \theta_{1,ij}} 
    + \sum_{j=1}^{d_2} \frac{\partial \bar y}{\partial b_{1,j}} \frac{\partial \bar y'}{\partial b_{1,j}}
    + \sum_{j=1}^{d_2} \frac{\partial \bar y}{\partial w_{2,j}} \frac{\partial \bar y'}{\partial w_{2,j}}\\
    %--------------------------
    &= \frac{c}{d_1d_2}\sum_{i=1, j=1}^{d_1,d_2} x_{1, i}x'_{1, i}w_{2, j}^2 \tilde\act'(\nu_{1,j})\tilde\act'_2(\nu'_{1,j})\\
    &\quad
    + \frac{\beta^2}{d_2} \sum_{j=1}^{d_2}  \tilde\act'(\nu_{1,j})\tilde\act'(\nu'_{1,j})
    + \frac{1}{d_2} \sum_{j=1}^{d_2} \tilde\act(\nu_{1,j})\tilde\act(\nu'_{1,j})\\
    %---------------------------------
    &= \frac{c}{d_1d_2}\sum_{i=1}^{d_1} x_{1, i}x'_{1, i}
    \sum_{j=1}^{d_2} w_{2,j}^2 \tilde\act'(\nu_{1,j})\tilde\act'(\nu'_{1,j})\\
    &\quad 
    + \frac{\beta^2}{d_2} \sum_{j=1}^{d_2} w_{2,j}^2\tilde\act'(\nu_{1,j})\tilde\act'(\nu'_{1,j})
    + \frac{1}{d_2} \sum_{j=1}^{d_2} \tilde\act(\nu_{1,j})\tilde\act(\nu'_{1,j})\\
    %----------------------------------
    &= (\frac{c}{d}\langle x, x' \rangle + \beta^2) \E[\tilde\act'(\nu)\tilde\act'(\nu')] 
    + \E[\tilde\act(\nu)\tilde\act(\nu')]\quad a.s.
\end{eqal*}
where $(\nu, \nu')$ has the same distribution as $(\nu_{2, j}, \nu'_{2, j})$ for any $j$. We make use of the fact $\E[w_{2,j}^2]=1$, and from central limit theorem, $x_{1, i}, x'_{1, i}$ and $\mu_{1, i},\mu_{1,i}'$ converge to joint Gaussian distribution for any fixed $x, x'$ as $d_1 \rightarrow \infty$
\begin{eqal*}
    \E[x_{1, i}x'_{1, i}] 
    &= \frac{1}{d}\E[\sum_{k=1}^d w_{ki}x_k \sum_{k'=1}^d w_{k'i} x'_{k'} ]\\
    &= \frac{1}{d}\E[\sum_{k=1}^d w_{ki}^2 x_k x'_k ]\\
    & = \frac{1}{d} \langle x, x' \rangle
\end{eqal*}
Similarly, 
\begin{eqal*}
    \E[\mu_{1, i}^2] & = \frac{c}{d_1}\sum_{i=1}^{d_1} \E[\theta_{1, ij}^2] \E[x_{1,j}^2]
    = \frac{c}{d}\Var[\theta]\\
    \E[\mu_{1, i}\mu_{1,i}'] & = \frac{c}{d_1}\sum_{i=1}^{d_1} \E[\theta_{1, ij}^2] \E[x_{1,j}x'_{1,j}]
    = \frac{c}{d}\Var[\theta]\langle x, x' \rangle\\
\end{eqal*}
\end{proof}

\subsection{Proof of Theorem \ref{thm:lazy3}}
\label{sec:prooflazy3}

Remind that as is proved in Theorem \ref{thm:lazy2}, $\varsigma_{1, j}(T) \rightarrow \varsigma_{1, t}(0)$ for any $T$ satisfying a mild condition, and $\varsigma_{1, t}(0)$ is nonzero almost surely. Making use the fact that $\tilde\act(\cdot; \varsigma)$ is continuous with respect to $\varsigma$, and its first and second order derivative is stochastically bounded,
the change of kernel $\cK$ induced by $\varsigma_{1, j}$ converges to 0 as $d_1, d_2 \rightarrow \infty$. This reduces to this quasi neural network to a standard neural network with activation function $\tilde\act(\cdot)$, which is twice differentiable and has bounded second order derivative. From Theorem 2 in \cite{jacot2018neural}, the kernel during training converges to the one during initialization. For the ease of the readers, we restate the proof below. On the limit $d_2\rightarrow \infty, d_1\rightarrow \infty$, 
% \begin{eqal*}
%     \frac{\partial \cK(x, x')}{\partial t}
%     &= \Big(\frac{c}{d_1d_2}\sum_{i=1}^{d_1} x_{1, i}x'_{1, i} + \frac{\beta^2}{d_2}\Big)
%     \frac{\partial}{\partial t}\sum_{j=1}^{d_2} w_{2,j}^2(t) \tilde\act'(\nu_{1,j}(t))\tilde\act'(\nu'_{1,j}(t))
%     + \frac{1}{d_2} \frac{\partial}{\partial t}\sum_{j=1}^{d_2} \tilde\act(\nu_{1,j})\tilde\act(\nu'_{1,j})\\
%     %------------------------
%     &= \frac{1}{d_2}(c\langle x, x'\rangle +\beta^2) 
%     \sum_{j=1}^{d_2} \Big(2w_{2,j}(t)\dot w_{2,j}(t)\tilde\act'(\nu_{1,j}(t))\tilde\act'(\nu'_{1,j}(t)) \\
%     &\quad + w_{2,j}^2(t)\tilde\act''(\dot\nu_{1,j}(t))\tilde\act'(\nu'_{1,j}(t)) + w_{2,j}^2(t)\tilde\act'(\nu_{1,j}(t))\tilde\act''(\dot\nu'_{1,j}(t))\Big)\\
%     &\quad + \frac{1}{d_2} \frac{\partial}{\partial t}\sum_{j=1}^{d_2} \Big(\tilde\act'(\dot\nu_{1,j})\tilde\act(\nu'_{1,j}) + \tilde\act(\nu_{1,j})\tilde\act'(\dot\nu'_{1,j}) \Big)\\
% \end{eqal*}
\begin{eqal*}
    &\cK(x, x')(t) - \cK(x, x')(0)\\
    &= \frac{c\langle x, x'\rangle +\beta^2}{d_2}
    \sum_{j=1}^{d_2} \big(w_{2,j}^2(t) \tilde\act'(\nu_{1,j}(t))\tilde\act'(\nu'_{1,j}(t)) - w_{2,j}^2(0) \tilde\act'(\nu_{1,j}(0))\tilde\act'(\nu'_{1,j}(0))\big)\\
    &\quad + \frac{1}{d_2} \sum_{j=1}^{d_2} \big(\tilde\act(\nu_{1,j}(t))\tilde\act(\nu'_{1,j}(t)) - \tilde\act(\nu_{1,j}(0))\tilde\act(\nu'_{1,j}(0))\big)\\
    %-----------------------------
    &= \frac{c\langle x, x'\rangle +\beta^2}{d_2}\Bigg(\sum_{j=1}^{d_2}\big(w_{2,j}^2(t) - w_{2,j}^2(0)\big)\tilde\act'(\nu_{1,j}(t))\tilde\act'(\nu'_{1,j}(t))\\
    &\quad + \sum_{j=1}^{d_2}w_{2,j}^2(0)\big(\tilde\act'(\nu_{1,j}(t))-\tilde\act'(\nu_{1,j}(0))\big) \tilde\act'(\nu'_{1,j}(t))\\
    &\quad + \sum_{j=1}^{d_2}w_{2,j}^2(0)\tilde\act'(\nu_{1,j}(0))\bigl(\tilde\act'(\nu'_{1,j}(t))-\tilde\act'(\nu'_{1,j}(0))\bigr)\Bigg),\\
    %-------------------------------
\end{eqal*}
\begin{eqal*}
    %&\leq \frac{c\langle x, x'\rangle +\beta^2}{d_2}
    |\cK(x, x')(t) - \cK(x, x')(0)|
    &\leq \Bigg| \frac{c\langle x, x'\rangle +\beta^2}{d_2} \Bigg| \\
    &\Bigg(\sum_{j=1}^{d_2}|w_{2,j}(t) - w_{2,j}(0)||w_{2,j}(t) + w_{2,j}(0)||\tilde\act'(\nu_{1,j}(t))\tilde\act'(\nu'_{1,j}(t))|\\
    &\quad + \sum_{j=1}^{d_2}w_{2,j}^2(0)\tilde\act'(\nu'_{1,j}(t)) \big|\tilde\act'(\nu_{1,j}(t))-\tilde\act'(\nu_{1,j}(0))\big| \\
    &\quad + \sum_{j=1}^{d_2}w_{2,j}^2(0)\tilde\act'(\nu_{1,j}(0))\bigl|\tilde\act'(\nu'_{1,j}(t))-\tilde\act'(\nu'_{1,j}(0))\bigr|\Bigg)\\
    &\quad + \frac{1}{d_2} \sum_{j=1}^{d_2} \big(\tilde\act(\nu_{1,j}(t))(\tilde\act(\nu'_{1,j}(t)) \tilde\act(\nu'_{1,j}(0)))\big)\\
    &\quad + \frac{1}{d_2} \sum_{j=1}^{d_2} \big(\tilde\act(\nu'_{1,j}(0))(\tilde\act(\nu_{1,j}(t)) - \tilde\act(\nu_{1,j}(0)))\big)\\
\end{eqal*}
From Theorem \ref{thm:lazy2}, and observing that $\tilde\act'(x), \tilde\act''(x)$ are bounded by constants, one can verify that each summation term is stochastically bounded by $\sqrt{d_2}$, so as $d_2 \rightarrow \infty$, $\cK(t)-\cK(0)$ converges to 0 at rate $\sqrt{d_2}$.

\subsection{Spherical harmonics decomposition to activation function}
Following \citet{bach2017breaking}, we start by studying the decomposition of action in quasi neural network (\ref{eq:quasiact}) and its gradients (\ref{eq:quasiactgrad}): for arbitrary fixed $c > 0$, $-1 \leq t \leq 1$, we can decompose equation (\ref{eq:quasiact}) and (\ref{eq:quasiactgrad}) as
%The gradient of activation function in the quasi neural network as defined in equation~\ref{eq:quasiact} can be decomposed as 
\begin{align}
    \tilde\sigma(ct) &= \sum_{k=0}^\infty \lambda_k N(d, k) P_k(t)\label{eq:actdecom},\\
    \tilde\sigma'(ct) &= \sum_{k=0}^\infty \lambda'_k N(d, k) P_k(t)\label{eq:actgraddecom},
\end{align}
where $P_k$ is the $k$-th Legendre polynomial in dimension $d$. 
% and it's relationship with spherical harmonics is
% \begin{equation*}
%     \sum_{j=1}^{N(p, k)} Y_{k, j}(x) Y_{k, j}(x') = N(p, k) P_k(\langle x, x' \rangle)
% \end{equation*}

\begin{lemma}
\label{lemma:quasiact}
The decomposition of activation function in the quasi neural network (\ref{eq:actdecom}) satisfies
\begin{enumerate}
    \item $\lambda_k = 0$ if $k$ is odd,
    \item $\lambda_k > 0$ if $k$ is even,
    \item $\lambda_k \asymp \mathrm{Poly}(k)(C/\sqrt{k})^{-k}$ as $k \rightarrow \infty$ when $k$ is even, where $\mathrm{Poly}(k)$ denotes a polynomial of $k$, and $C$ is a constant.
    %\item $\lambda_k \sim \left(\frac{Ck}{d^2}\right) ^{k/2}$ when $1 \ll k \ll d$ and when $k$ is even.
\end{enumerate}
Its gradient (\ref{eq:actgraddecom}) satisfies
\begin{enumerate}
    \item $\lambda'_k = 0$ if $k$ is even,
    \item $\lambda'_k > 0$ if $k$ is odd,
    \item $\lambda'_k \asymp \mathrm{Poly}(k)(C/\sqrt{k})^{-k}$ as $k \rightarrow \infty$ when $k$ is odd, where $\mathrm{Poly}(k)$ denotes a polynomial of $k$, and $C$ is a constant.
    %\item $\lambda_k \sim \left(\frac{Ck}{d^2}\right) ^{k/2}$ when $1 \ll k \ll d$ and when $k$ is odd.
\end{enumerate}
\end{lemma}

%\subsection{Proof of Lemma \ref{lemma:quasiact}}
\begin{proof}

Let's start with the derivative of activation function in quasi neural network:
\begin{equation*}
    \tilde\sigma'(t) = \Phi(\hat ct), -1 \leq t \leq 1,
\end{equation*}
where $\hat c$ is a constant. We introduce the auxiliary parameters $x, w \in \R^d$ s.t. $\|x\|_2=\|w\|_2=1$ and let $t = w^T x$ By Cauchy–Schwarz inequality, $-1\leq w^T x \leq 1$. 
Following \cite{bach2017breaking}, %for any function that can be decomposed as $g(x) = \int_{\S^{d-1}}p(\theta)\tilde\sigma'_{\|\theta\|}(\theta^Tx)d\tau(\theta)$, it can be decomposed 
we have the following decomposition to $\tilde\sigma'(w^T x)$:
\begin{equation*}
    \tilde\sigma'(w^T x) = \sum_{k=1}^\infty \lambda_k' N(d, k) P_k(w^T x),
\end{equation*}
%we have
%we will study the sphere harmonic decomposition of $\sigma'(\hat \theta^T x)$ on unit sphere $x \in \R^{d+1}\|x\|_2 = 1$. For the unit sphere $\mathbb S^d = \{x \in \R^{d+1}: \|x\|_2=1\}$ function $g(x)$ that can be decomposed as $g(x)=\int_{\mathbb S^d} p(\hat\theta)\sigma(\hat\theta^Tx)d\tau_d(\hat\theta)$, then we can decompose it as
% $$
% g_k(x) = \lambda_k p_k(x)
% $$
where $N(d, k)$ and $P_k(\cdot)$ are defined in section \ref{sec:sh}, $\lambda_k'$ can be computed by
$$
\begin{aligned}
\lambda'_k &= \frac{w_{d-1}}{w_d}\int_{-1}^{1} \tilde\sigma'(t)P_k(t)(1-t^2)^{(d-2)/2}dt\\
&= \left(-\frac{1}{2}\right)^k\frac{\Gamma((d-1)/2)}{\Gamma(k+(d-1)/2)}\frac{w_{d-1}}{w_d} \int_{-1}^{1} \tilde\sigma'(t)\left(\frac{d}{dt}\right)^k(1-t^2)^{k+(d-3)/2}dt.
\end{aligned}
$$
To solve this itegration, we can apply Taylor decomposition to $\tilde\sigma'(\cdot)$:
\begin{equation}
    \tilde\sigma'(ct) = \frac{1}{2} + \frac{1}{\sqrt{2\pi}}\sum_{n=0}^\infty \frac{(-1)^n\hat c^{2n+1}}{2^n n!(2n+1)}t^{2n+1}.
    \label{eq:tayler}
\end{equation}

%where $c=\sqrt{\frac{6}{d}}\frac{\|\theta\|_2}{\varsigma_1}$.
We will study the following polynomial integration first
$$
\int_{-1}^{1} t^\alpha\left(\frac{d}{dt}\right)^k(1-t^2)^{k+(d-3)/2}dt.
$$
When $\alpha < k$, this integration equals 0 as $P_k$ is orthogonal to all polynomials of degree less than $k$. If $(\alpha-k) \mod 2 \neq 0$, this integration is 0 because the function to be integrated is an odd function. 
%When $k$ is even, $\lambda_k=0$. 
For $\alpha \geq k$ and $k \equiv \alpha \mod 2$ ($k$ is odd), using successive integration by parts,
\begin{equation}
\begin{aligned}
    \int_{-1}^{1} t^\alpha\left(\frac{d}{dt}\right)^k(1-t^2)^{k+(d-3)/2}dt
& = (-1)^k\frac{\alpha!}{(\alpha-k)!}\int_{-1}^{1}t^{\alpha-k}(1-t^2)^{k+(d-3)/2}dt\\
& = (-1)^k\frac{\alpha!}{(\alpha-k)!}\int_{-\pi/2}^{\pi/2}\sin^{\alpha-k}(x)\cos^{2k+(d-2)}(x)dx\\
& = (-1)^kC_d\frac{\alpha!(2k+d-3)!!}{(\alpha-k)!!(\alpha+k+d-2)!!},
\end{aligned}
\label{eq:spint}
\end{equation}
%where $C_d=2$ if $d$ is even, and $C_d = \frac{\pi}{2}$ if $d$ is odd.
where $C_d$ is a constant that depends only on $d$ mod 2.

Combining (\ref{eq:tayler}) and (\ref{eq:spint}), we have $\lambda_k=0$ when $k$ is even and $k \neq 0$. When $k$ is odd,

\begin{eqal*}
\lambda_k' 
%&=\left(-\half\right)^k\frac{\Gamma(d/2)}{\Gamma(k+d/2)}\frac{w_{d-1}}{w_d} \frac{C_d}{\sqrt{2\pi}} \sum_{\alpha=k:2}^\infty c^\alpha (-1)^{(\alpha-1)/2}
%\frac{(\alpha-1)!(\alpha-k-1)!!(2k+d-2)!!}{(\alpha-1)!!(\alpha-k)!(\alpha+k+d-1)!!} \\
=& \left(-\half\right)^k\frac{\Gamma((d-1)/2)}{\Gamma(k+(d-1)/2)}\frac{w_{d-1}}{w_d}\frac{C_d}{\sqrt{2\pi}}
 \sum_{\alpha=k:2}^\infty
\hat c^\alpha (-1)^{(\alpha-1)/2}
\frac{(\alpha-2)!!(2k+d-3)!!}{ (\alpha-k)!!(\alpha+k+d-2)!!}.
\end{eqal*}
%For two layer neural network, $\sigma$ contains only the ``induraced variance'' from the first layer, and $\sqrt{d}\sigma$ converges to $\frac{1}{\sqrt{3}}$ at rate $O(1/\sqrt{d})$. On the other hand, $c = \sqrt{\frac{6}{d}}\|\theta\|_2$ converges to $\sqrt{2}$ at rate $O(1/\sqrt{d})$ according to CLT.

Following \cite{bach2017breaking, geifman2020similarity} we take $d$ as a constant and take $k$ to infinity.
Let $\beta = (\alpha -k)/2 \geq 0$ we have 

\begin{eqal*}
\lambda_k' &= (-1)^{(k+1)/2}\left(\half\right)^k\frac{\Gamma((d-1)/2)}{\Gamma(k+(d-1)/2)}\frac{w_{d-1}}{w_d}\frac{C_d}{\sqrt{2\pi}} \sum_{\beta=0}^\infty \frac{(-1)^\beta \hat c^{2\beta+k}(2\beta+k-2)!!(2k+d-3)!!}{  (2\beta)!!(2\beta+2k+d-2)!!}\\
%--------------------------------------
&\asymp (-1)^{(k+1)/2}\left(\half\right)^k\frac{\Gamma((d-1)/2)}{\Gamma(k+(d-1)/2)}\frac{w_{d-1}}{w_d}\frac{C_d}{\sqrt{2\pi}} \sum_{\beta=0}^\infty \frac{(-1)^\beta \hat c^{2\beta+k}\Gamma(\beta+k/2)\Gamma(k+(d-1)/2)}{  \beta!\Gamma(\beta+k+d/2)2^{\beta -k/2}}\\
%--------------------------------------
&:= (-1)^{(k+1)/2}\left(\half\right)^k\frac{\Gamma((d-1)/2)}{\Gamma(k+(d-1)/2)}\frac{w_{d-1}}{w_d}\frac{C_d}{\sqrt{2\pi}} \sum_{\beta=0}^\infty g(\beta, k).
\end{eqal*}
where $\asymp$ means the radio converge to a constant which doesn't depend on $k$ or $\beta$ as $k \rightarrow \infty$. Here we introduced the function $g(\beta, k)$ for simplification, and it satisfies
%When $k \gg c/\sigma$, $g(\beta, k)$ converges with $\beta$ to 0 at super-exponential rate and the first few terms dominates $\lambda_k$, so we only need to consider the case $k \gg \beta$. In this regime, $g(\beta, k) \sim $
%In this case, each term converges at rate $(2\sigma/c)^{-k} k^{-k/2}$, so the summation converges at rate $(2\sigma/c)^{-k} k^{-k/2}$.
\begin{eqal*}
\frac{g(\beta, k)}{g(\beta-1, k)} =-\frac{\hat c^2(2\beta+k-3)}{2\beta(2\beta + 2k + d - 2)},
\end{eqal*}
which indicates that $g(\beta, k)$ decays at factorial rate when $\beta > \hat c^2/2$.
%On the other hand, when $k \gg \beta, k \gg d$, 
If $k \gg \hat c^2/2$, $\beta \ll k$ regime dominates the summation. 
%We further study the regime that $k \gg d$.
% Using Stirling's approximation, 
% $$
% m!!\asymp \sqrt{m}(m/e)^{\frac{m}{2}},
% $$
% where $C_m=1$ when $m$ is even, and $C_m=\sqrt{2}$ when $m$ is odd, $a \asymp b$ means $\frac{a}{b}$ converges to a constant as $k \rightarrow \infty$,\kz{Is anyone willing to varify the following 3 equations?}
% \begin{eqal*}
% g(\beta, k) &\asymp 
% %\frac{(k/e)^{\beta + k/2} (2k/e)^k e^\beta \sqrt{k}}{(\sigma/c)^{(2\beta+k)}(2\beta/e)^\beta (2k/e)^{k+\beta} e^\beta}\\
% %&= \left(\sqrt{\frac{k}{e}}\frac{c}{\sigma}\right)^k\left(\frac{4\beta}{e} \frac{\sigma}{c}\right)^{-\beta}\sqrt{k}
% (-1)^\beta \hat c^{2\beta+k} \left(\frac{k}{2\beta}\right)^\beta \left(\frac{k}{e}\right)^{\frac{k-2}{2}}  \left(\frac{e}{2k+d}\right)^{\beta+\half}\sqrt{\frac{k}{2\beta e}}.%\\
% %&\sim \left(c\sqrt{\frac{k}{e}}\right)^k \frac{1}{2k} \left(-\frac{ec^2}{4\beta}\right)^\beta \sqrt{\frac{e}{\beta}}
% \end{eqal*}

Using Stirling's approximation, one can easily prove
\begin{eqal*}
    \Gamma(k+x) \asymp \Gamma(k) k^x
\end{eqal*}
When $k \gg d$, 
\begin{eqal*}
g(\beta, k) &=
% \left(\hat c\sqrt{\frac{k}{e}}\right)^k \frac{1}{k} \left(-\frac{e\hat c^2}{4\beta}\right)^\beta \sqrt{\frac{e}{\beta}}.
\frac{(-1)^\beta \hat c^{2\beta+k}\Gamma(\beta+k/2)\Gamma(k+(d-1)/2)}{  \beta!\Gamma(\beta+k+d/2)2^{\beta -k/2}}\\
%------------------------------
&\asymp \left(-\frac{1}{4}\right)^\beta \hat c^{2\beta+k}\Gamma(k+(d-1)/2) \frac{2^{k/2}\Gamma(k/2)}{\Gamma(k) k^{d/2}\beta!}\\
&=  \hat c^{k}\Gamma(k+(d-1)/2) \frac{2^{k/2}\Gamma(k/2)}{\Gamma(k) k^{d/2}} \left(-\frac{\hat c^2}{4}\right)^\beta \frac{1}{\beta!}
\end{eqal*}

This splits $g(\beta, k)$ into two parts: the first part 
%$\left(\hat c\sqrt{\frac{k}{e}}\right)^k \frac{1}{2k}$ 
depends only on $k$ and the rest part only depends on $\beta$. The summation of the second part over $\beta$ yields
\begin{eqal*}
    %\sum_{\beta=0}^\infty \left(-\frac{e\hat c^2}{4\beta}\right)^\beta \sqrt{\frac{e}{\beta}} 
    %& \asymp 
    \sum_{\beta=0}^\infty \left(-\frac{\hat c^2}{4}\right)^\beta \frac{1}{\beta!} 
    = \exp(-\frac{\hat c^2}{4}), 
    %= \exp(-\frac{3\|\theta\|_2}{2d\varsigma_1})
\end{eqal*}
%when $\beta > \frac{ce}{\sigma}\sim \sqrt{d}$, $g(\beta, k)$ is decreasing with $\beta$. 
%If $k \gg O(d)$, the terms with $\beta \ll k$ dominates the summation, so 
Using Stirling's approximation
\begin{equation*}
    \gamma(x+1) \asymp \sqrt{2\pi x}(x/e)^x,
\end{equation*}
this leads to the expression for $\lambda_k$:
\begin{eqal*}
    \lambda_k' &\asymp 
    (-1)^{(k+1)/2}\left(\half\right)^k\frac{\Gamma((d-1)/2)}{\Gamma(k+(d-1)/2)}\hat c^{k}\Gamma(k+(d-1)/2) \frac{2^{k/2}\Gamma(k/2)}{\Gamma(k) k^{d/2}}\exp(-\frac{\hat c^2}{4})\\
    &\asymp (-1)^{(k+1)/2}\left(\frac{\hat c}{2}\right)^k \frac{2^{k/2}\Gamma(k/2)}{\Gamma(k) k^{d/2}}\exp(-\frac{\hat c^2}{4})\\
    & \asymp (-1)^{(k+1)/2}\left(\frac{\hat c}{2}\sqrt\frac{e}{k}\right)^k k^{-d/2}\exp(-\frac{\hat c^2}{4})
\end{eqal*}

% $$
% \begin{aligned}
% \lambda_k' \asymp  (-1)^{(k+1)/2}\frac{\Gamma((d-1)/2)}{\Gamma(k+(d-1)/2)}\frac{1}{2k}\left(\frac{\hat c}{2}\sqrt{\frac{k}{e}}\right)^k  
% \exp(-\frac{\hat c^2}{4})
% %\exp(-\frac{3\|\theta\|_2}{2d\varsigma_1})
% \asymp (-1)^{(k+1)/2}{k^{-\frac{d-1}{2}}}\left(\frac{\hat c}{2}\sqrt{\frac{1}{ek}}\right)^k  \exp(-\frac{\hat c^2}{4}).
% %A_d (B_d k)^{(-k+1)/2}
% \end{aligned}
% $$
%where $A_d$ and $B_d$ are constants dependent only on $d$.

% On the other hand, when $\frac{\hat c^2}{2} < k \ll d$, 
% $$
% g(\beta, k) \asymp 
% \left( \frac{\hat c^2(k-1)}{e}\right)^{\frac{k-2}{2}} \sqrt{\frac{k-1}{d}} \left( -\frac{\hat c^2k}{2d} \frac{1}{\beta} \right)^\beta \frac{1}{\sqrt{\beta}}.
% $$
% The summation over the second part
% \begin{equation*}
% \begin{aligned}
%     \sum_{\beta=0}^\infty \left( -\frac{\hat c^2k}{2d} \frac{1}{\beta} \right)^\beta \asymp \exp(-\frac{\hat c^2k}{2d}),
% \end{aligned}
% \end{equation*}
% so
% \begin{eqal*}
% \lambda_k' \asymp  (-1)^{(k+1)/2}\frac{\Gamma(d/2)}{\Gamma(k+d/2)} \left( \frac{\hat c^2(k-1)}{e} \right)^{\frac{k-2}{2}} \sqrt{\frac{k-1}{d}} 
% \asymp 
% (-1)^{k/2}k^{-\frac{d+4}{2}}\left( \tilde c\|w\|\sqrt{\frac{1}{ek}} \right)^k\exp\left(-\frac{\|w\|^2}{4}\right).
% %\left( \frac{2}{d} \right)^k \left( \frac{\hat c^2(k-1)}{e} \right)^{\frac{k-2}{2}} \sqrt{\frac{k-1}{d}}. 
% \end{eqal*}

Similarly, the activation function of quasi neural network has the Tayler expansion
\begin{equation*}
\begin{aligned}
    \tilde\sigma(x)&=\varsigma_\ell\varphi\left(\hat c t\right) + x\Phi\left(\hat c t\right)
    \label{eq:quasiactdecom}\\
    &= \frac{t}{2} + \sum_{n=0}^\infty \frac{(-1)^n \hat c^{2n+1}}{2^{n+1}(n+1)!(2n+1)} t^{2n+2}.
\end{aligned}
\end{equation*}

So $\lambda_k=0$ when $k$ is odd, and when $k$ is even:
\begin{eqal*}
    \lambda_k &= (-1)^{(k+1)/2}\left(\half\right)^k\frac{\Gamma((d-1)/2)}{\Gamma(k+(d-1)/2)}\frac{w_{d-1}}{w_d}\frac{C_d}{\sqrt{2\pi}} \sum_{\beta=0}^\infty \frac{(-1)^\beta \hat c^{2\beta+k}(2\beta+k-2)!!(2k+d-3)!!}{  (2\beta)!!(2\beta+2k+d-2)!!}
\end{eqal*}

Furthermore, when $k \gg d$,
$$
\lambda_k \asymp  (-1)^{k/2}{k^{-\frac{d}{2}}}\left(\frac{\hat c}{2}\sqrt{\frac{e}{k}}\right)^k \exp\left( -\frac{\hat c^2}{4} \right),
$$
\end{proof}

\label{sec:lambda}

%At this stage, we are re

\subsection{Computing covariance matrix}
In this part, 
we prove Theorem \ref{thm:exp} by computing $\Sigma^{(0)}$ and $\Sigma^{(1)}$.

\textbf{Theorem \ref{thm:exp}} \textit{NTK of a binary weight neural network can be decomposed using equation (\ref{eq:spdecom}). If $k \gg d$, then
\begin{equation*}
\mathrm{Poly}_1(k)(C)^{-k} \leq u_k \leq \mathrm{Poly}_2(k) (C)^{-k}
\end{equation*}
where $\mathrm{Poly}_1(k)$ and $\mathrm{Poly}_2(k)$ denote polynomials of $k$, and $C$ is a constant.}

%we need to compute $\Sigma^{(0)}$ and $\Sigma^{(1)}$. 
We make use of the results in Section \ref{sec:lambda}, and remind that $\lambda_k, \lambda'_k$ depends on $\hat c$, we make this explicit as $\lambda_k(\hat c), \lambda'_k(\hat c)$.
We introduce an auxiliary parameter $w \sim \cN(0, I)$, and denote $\tilde c=\sqrt\frac{c\Var[\Theta]}{d\ \tilde \varsigma^2}=\sqrt\frac{\Var[\theta]}{1-\Var[\theta]}, \tilde w = w / \|w\|_2$, then the decomposition of kernel (\ref{eq:spdecom}) can be computed by 

\begin{equation*}
\begin{aligned}
    \Sigma^{(1)} 
    &= \E_{\theta}\left[\tilde\sigma\left(\mu \right) \tilde\sigma\left(\mu\right)\right]\\
    &= \E_{w \sim \cN(0, I)}\left[\tilde\sigma(\tilde c\langle w, x \rangle) \tilde\sigma(\tilde c\langle w, x' \rangle)\right]\\
    %&= \E_{\|w\|} \E_{\tilde w}\left[\tilde\sigma(\langle \tilde w, x \rangle) \tilde\sigma(\langle \tilde w, x' \rangle)\right]\\
     &=  \E_{\|w\|} \int \tilde\sigma(\tilde c\langle \tilde w, x \rangle) \tilde\sigma(\tilde c\langle \tilde w, x' \rangle) d\tau(\tilde w)\\
     &= \E_{\|w\|} \sum_{k=0}^\infty (\lambda_k(\tilde c \|w\|))^2 N(p, k) P_k(\langle x, x' \rangle),\\
     %-----------------------------------
     \Sigma^{(0)} 
     &=\E_{\theta}\left[\tilde\sigma'\left(\mu \right) \tilde\sigma'\left(\mu\right)\right]\\
     &= \E_{\|w\|} \sum_{k=0}^\infty {(\lambda_k'(\tilde c \|w\|)})^2 N(p, k) P_k(\langle x, x' \rangle).
\end{aligned}    
\end{equation*}

First compute $\Sigma^{(0)}$. According to Lemma 16 in \citet{bietti2019inductive},
$$
u_{0, k} = \E_{w \sim \cN(0, I)}[{\lambda_k'}^2] = \E_{\|w\|}[{\lambda_k'}^2].
$$

%See the appendix for more detail.

%\subsection{Proof of Theorem \ref{thm:exp}}
% \textcolor{red}{Here}
% \begin{equation}
% \begin{aligned}
%     u_{0, k} 
%     &= \E_{\|w\|_2}\left[\lambda_k^2\right] \\
%     &= \E_{\|w\|_2} k^{-d-4}(a\|w\|_2^2)^{k}k^{-k}\exp(-a\|w\|_2^2/2),
% \end{aligned}   
% \end{equation}

%where $a$ is a constant, $\|w\|_2^2 \sim \chi^2(d)$ which can be approximated by $\cN(d, (\sqrt{2d})^2)$.

% As is shown in Section \ref{sec:ntk},
% \begin{equation*}
% \begin{aligned}
%     u_{0, k} &= \E_{\|\theta\|} \lambda_k^2\\
%     &\propto k^{-d-3}\left(\frac{3\|\theta\|^2}{2\varsigma_1^2d}\frac{1}{ek} \right)^k
% \end{aligned}
% \end{equation*}
Remind that
$$
\lambda'_k(\tilde c\|w\|) \asymp (-1)^{k/2}k^{-d/2}\left( \frac{\tilde c\|w\|}{2}\sqrt{\frac{e}{k}} \right)^k\exp\left(-\frac{\tilde c^2\|w\|^2}{4}\right).
$$ 

\begin{eqal*}
    u_{0, k}
    &= \E_{\|w\|} {(\lambda_k'(\tilde c \|w\|)})^2\\
    &\asymp \E_{\|w\|}k^{-d}\left( \frac{\tilde c^2\|w\|^2e}{4k} \right)^k\exp\left(-\frac{\tilde c^2\|w\|^2}{2}\right)\\
    &= k^{-d}\left( k/e \right)^{-k} \E_{\tilde c\|w\|}\left(\frac{\tilde c^2\|w\|^2}{4}\right)^k\exp\left(-\frac{\tilde c^2\|w\|^2}{2}\right)  
\end{eqal*}

Because $w\sim \cN(0, 1)$, $\|w\|_2^2$ satisfy Chi-square distribution, and its momentum generating function is 
\begin{eqal*}
    M_X(t) &=\E[\exp(t\|w\|^2)] = (1-2t)^{-d/2}
\end{eqal*}
It's $k$-th order derivative is 
\begin{eqal*}
    M_X^{(k)} &= \E[\|w\|^{2k}\exp(t\|w\|^2)] = \frac{(d+2k-2)!!}{(d-2)!!}(1-2t)^{-\frac{d}{2}-k}
\end{eqal*}
Let $t = -\tilde c^2/2$, we get 
\begin{eqal*}
    \E\left[\|w\|^{2k}\exp\left(-\frac{\tilde c^2\|w\|^2}{2}\right)\right] &= \frac{(d+2k-2)!!}{(1+\tilde c^2)^{d/2+k}(d-2)!!} \\
    &\asymp 2^k \frac{\Gamma(k+d/2)}{\Gamma(d/2)} (1+\tilde c^2)^{-k-d/2}\\
    &\asymp \left(\frac{2k}{(1+\tilde c^2)e}\right)^{d/2+k}\sqrt\frac{1}{k}
\end{eqal*}
so 
\begin{eqal*}
    u_{0, k} &\asymp  \left(\frac{\tilde c}{2}\right)^{2k} k^{-d}\left( k/e \right)^{-k} \left(\frac{2k}{(1+\tilde c^2)e}\right)^{d/2+k}\\
    & \asymp k^{-(d-1)/2}\left(\frac{\tilde c^2}{2(1+\tilde c^2)}\right)^k
    % &\asymp k^{-(d+9)/2}\left(\frac{c\Var[\Theta]}{d\ \tilde \varsigma^2e}\right)^{k}\\
    % &= k^{-(d+9)/2} \left(\frac{\Var[\theta]}{e(1-\Var[\theta])}\right)^k
\end{eqal*}
when $k$ is odd, and 0 when $k$ is even.

Similarly, 
\begin{eqal*}
    u_{1, k} &\asymp  \left(\frac{\tilde c}{2}\right)^{2k} k^{-d}\left( k/e \right)^{-k} \left(\frac{2k}{(1+\tilde c^2)e}\right)^{d/2+k}\\
    & \asymp k^{-(d-1)/2}\left(\frac{\tilde c^2}{2(1+\tilde c^2)}\right)^k
    % &\asymp k^{-(d+9)/2}\left(\frac{c\Var[\Theta]}{d\ \tilde \varsigma^2e}\right)^{k}\\
    % &= k^{-(d+9)/2} \left(\frac{\Var[\theta]}{e(1-\Var[\theta])}\right)^k
\end{eqal*}
% \begin{eqal*}
%     u_{1, k} &\asymp \tilde c^{2k} k^{-(d+3)}\left( ek \right)^{-k} \left(\frac{k}{e}\right)^{d/2+k-1/2}\\
%     &\asymp k^{-(d+7)/2}\left(\frac{c\Var[\Theta]}{d\ \tilde \varsigma^2e}\right)^{k}\\
%     &= k^{-(d+7)/2} \left(\frac{\Var[\theta]}{e(1-\Var[\theta])}\right)^k
% \end{eqal*}
when $k$ is even, and 0 when $k$ is odd.

Finally, using the recurrence relation 
\begin{eqal*}
    tP_k(t) = \frac{k}{2k+d-3}P_{k-1}(t) + \frac{k+d-3}{2k+d-3}P_{k+1}(t)
\end{eqal*}
taking them into \eqref{eq:recntk} finishes the proof.

\subsection{Gaussian kernel}
\begin{eqal*}
    \cK_{RGauss}(x, x') &= \E[\cK_{Gauss}(\kappa x, \kappa x')]\\
    &= \E\left[\exp\left( -\frac{\kappa^2\|x-x'\|}{\xi^2}\right)\right] \\
    &= \E\left[\exp\left( -\frac{\|x-x'\|}{(\xi/\kappa)^2}\right)\right] 
\end{eqal*}
This indicates that this kernel can be decomposed using spherical harmonics \eqref{eq:spdecom}, and when $k \gg d$, the coefficient 
\begin{eqal*}
    u_k &= \E\left[\exp\left(-\frac{2\kappa^2}{\xi^2}\right)\left(\frac{\xi}{\kappa}\right)^{d-2}I_{k+d/2-1}\left(\frac{2\kappa^2}{\xi^2}\right) \Gamma\left(\frac{d}{2}\right)\right]\\
    &\asymp \E\left[\exp\left(-\frac{2\kappa^2}{\xi^2}\right)\Gamma\left(\frac{d}{2}\right)\sum_{j=0}^\infty \frac{1}{j!\Gamma(k+d/2+j)}\left(\frac{\kappa^2}{\xi^2}\right)^{k+2j}\right] \\
    &= \sum_{j=0}^\infty \frac{\Gamma(d/2)}{j!\Gamma(k+d/2+j)}
    \E\left[\left(\frac{\kappa^2}{\xi^2}\right)^{k+2j}\exp\left(-\frac{2\kappa^2}{\xi^2}\right)\right] \\
    &= \sum_{j=0}^\infty \frac{\Gamma(d/2)}{j!\Gamma(k+d/2+j)}\frac{\Gamma(k+2j+d/2)}{\Gamma(d/2)}
    \left(\frac{2}{\xi^2}\right)^{k+2j}\left(\frac{1}{1+4/\xi^2}\right)^{(k+2j+d/2)}\\
    &\asymp \left(\frac{2}{\xi^2}\right)^{k}\left(1+\frac{4}{\xi^2}\right)^{(-k-d/2)}\sum_{j=0}^\infty \frac{1}{j!}\left(\frac{k(2/\xi^2)^2}{(1+4/\xi^2)^2}\right)^j\\
    &\asymp \left(\frac{2}{4+\xi^2}\right)^{k}\exp\left(\left(\frac{2}{4+\xi^2}\right)^2k\right).
\end{eqal*}
Note that $\frac{2}{4+\xi^2}\exp\left(\left(\frac{2}{4+\xi^2}\right)^2\right)$ is always smaller than 1 so $u_k$ is always decreasing with $k$.

\section{Additional information about numerical result}
\label{sec:expadd}
\subsection{Toy dataset}
In neural networks (NN) experiment, we used three layers with the first layer fixed. The number of hidden neural is 512. In neural network with binary weights (BWNN) experiment, the setup is the same as NN except the second layer is Binary. We used BinaryConnect method with stochastic rounding. We used gradient descent with learning rate searched from $10^{-3}, 10^{-2}, 10^{-1}$. For Laplacian kernel and Gaussian kernel, we searched kernel bandwidth from $2^{-2}\mu$ to $2^2\mu$ by power of 2, and $\mu$ is the medium of pairwise distance. The SVM cost value parameter is from $10^{-2}$ to $10^4$ by power of 2.

More results are listed in Table \ref{tab:ucimore}. Accuracy are shown in the format of mean $\pm$ std. P90 and P95 denotes the percentage of dataset that a model achieves at least 90\% and 95\% of the highest accuracy, respectively.

\begin{table}
    \caption{More results in UCI dataset experiment.}
    \label{tab:ucimore}
    \begin{tabular}{c|ccc|ccc}
        \hline
        \multirow{2}{*}{Classifier} & \multicolumn{3}{c|}{Training}  & \multicolumn{3}{c}{Testing}\\
        & Accuracy & P90 & P95 &  Accuracy & P90 & P95\\
        \hline
        %NN & 96.19 $\pm$ 8.03 \% & 77.62 $\pm$ 16.10 \%\\
        % 85.00 pm 13.49 & 79.21 pm 15.47
        % NN & 88.73 $\pm$ 12.34\% & 81.09 $\pm$ 14.44\% \\
        % %BNN & 93.55 $\pm$ 10.38\% & 77.82 $\pm$ 16.97\%\\
        % BWNN & 86.41 $\pm$ 13.94\% & 80.57 $\pm$ 14.38\%\\
        % Laplacian & 93.52 $\pm$ 9.65\% & 81.61 $\pm$ 14.72\%\\
        % Gaussian & 91.07 $\pm$ 10.63\% & 81.39 $\pm$ 14.85\%\\
        NN & 96.19$\pm$8.03\% & 96.67\% & 91.11\% & 77.62$\pm$16.10\% & 73.33\% & 56.67\%\\
        BWNN & 93.55$\pm$10.39\%&84.44\%& 76.67\% & 77.83$\pm$16.57\% & 77.78\%& 54.44\% \\
        Laplacian & 93.52$\pm$9.65\% &85.56\% & 76.67\% & 81.62$\pm$14.72\% &97.78\% & 91.11\%\\
        Gaussian & 91.08$\pm$10.63\% & 76.67\% & 58.89\% & 81.40$\pm$14.85\% & 95.56\% & 87.78\%\\
        \hline
    \end{tabular}
\end{table}

\subsection{MNIST-like dataset}
Similar to the toy dataset experiment, we used three layer neural networks with the first layer fixed, and only quantize the second layer. The number of neurons in the hidden layer is 2048. The batchsize if 100 and ADAM optimizer with learning rate $10^{-3}$ is used. 

% \newpage
% \section{Update from previous submission}
% The previous (rejected) submission to ICML is a preliminary draft with the careless paper organizations and unclear statements. We made substantial improvements in the exposition and the technical contributions. We reorganized the paper to make the statements and the proofs easier to follow. We clarify the definition of quasi neural network by adding illustrative figures, and fix the technical mistakes in the original proofs. For the completeness of the study, we provided an additional experiment to support our theorem of approximation with quasi neural network.

\end{document}